\documentclass{article}
\usepackage[OT1]{fontenc}
\usepackage{iclr2027_conference,times}

\usepackage{amsmath,amsfonts,bm}

\def\eqref#1{equation~\ref{#1}}

\def\1{\bm{1}}

\DeclareMathAlphabet{\mathsfit}{\encodingdefault}{\sfdefault}{m}{sl}
\SetMathAlphabet{\mathsfit}{bold}{\encodingdefault}{\sfdefault}{bx}{n}

\newcommand{\Bridges}{\mathcal{B}}
\newcommand{\Components}{\mathcal{C}}

\usepackage{amsmath,amssymb,amsthm}
\usepackage[ruled,vlined,linesnumbered]{algorithm2e}
\usepackage{booktabs}
\usepackage{colortbl}
\usepackage{tabularx}
\usepackage{graphicx}
\usepackage{wrapfig}
\usepackage{float}
\usepackage{placeins}
\usepackage{microtype}
\usepackage{xcolor}
\usepackage{hyperref}
\usepackage{url}
\definecolor{ringtreelink}{HTML}{1F5A7A}
\definecolor{tablegray}{HTML}{ECECEC}
\definecolor{autoregressive}{HTML}{A3123F}
\definecolor{oneshot}{HTML}{1C88A8}
\newcolumntype{Y}{>{\raggedleft\arraybackslash}X}
\hypersetup{
  colorlinks=true,
  linkcolor=ringtreelink,
  citecolor=ringtreelink,
  urlcolor=ringtreelink
}

\graphicspath{{figures/}}

\newtheorem{proposition}{Proposition}
\newtheorem{lemma}{Lemma}
\newtheorem{theorem}{Theorem}

\title{TopoCompile: Molecular Generation\\with Exact Global Count Constraints}

\author{Xiang Wu \quad Rong-Hua Li \quad Quanqing Xu\\\textbf{Kangfei Zhao \quad Guoren Wang}}

\begin{document}

\maketitle

\begin{abstract}
Molecular generation designs novel molecules with desired structures and
properties. Global count constraints specify molecular size,
elemental composition, and ring content without prescribing detailed structure.
Property-conditioned generators steer generation toward such targets but do
not guarantee exact
satisfaction. We introduce \textsc{TopoCompile}, which combines an
autoregressive component-tree generator with a dynamic-programming compiler
for supported component-additive integer counts. \textsc{TopoCompile}
represents molecules as trees of intact ring systems and individual atoms,
making each supported count additive over components. At each step, an
autoregressive Transformer scores component and attachment actions, while the
compiler masks actions with no target-satisfying completion. Within the fixed
component vocabulary and typed attachment rules, constrained decoding
guarantees exact equality for supported non-negative, component-additive
integer counts. Applied at inference time, these checks let the same trained
generator support different targets and count combinations without retraining.
Across four benchmarks, \textsc{TopoCompile} achieves 100\% Exact-valid yield
for Common and Rare four-count target groups; the strongest baseline within
each dataset reaches 44.9\% or less on Rare targets. In unconditional
generation, it achieves the highest validity or valid-unique-novel (V.U.N.)
among the compared methods across all four benchmark datasets.

\end{abstract}

\section{Introduction}
\label{sec:introduction}

Designing molecules with specified structural and physicochemical properties is a
central problem in drug discovery and materials design. Neural molecular
generators address this problem by learning distributions over molecular
structures and using them to produce chemically plausible candidates beyond the
training set
\citep{gomez2018automatic,jin2018junction,zang2020moflow,shi2020graphaf,luo2021graphdf}.
Recent advances in diffusion, flow, and autoregressive models have further
improved molecular generation across graph and sequence representations
\citep{vignac2023digress,jolicoeurmartineau2025anyproperty,wang2025learningorder}.

In practical molecular design, chemically plausible candidates often also need to
meet specific structural requirements. Global count constraints provide a direct
way to express these requirements: users can specify molecular size, elemental
composition, and ring content without prescribing detailed connectivity. Such constraints arise in molecular structure elucidation, where the molecular formula is known but the atom connectivity remains unresolved
\citep{stravs2022msnovelist,bohde2025diffms}. They also support controlled comparison
libraries for virtual screening and molecular-model evaluation
\citep{imrie2021deepcoy}, while exact ring counts enable targeted exploration of
atypical topologies such as high-cycle-rank molecules. The generator chooses
detailed structures from the learned distribution while the user controls
global molecular characteristics.

Existing approaches differ in how precisely they control generation and what
users must specify. Property-conditioned generators steer generation toward
multiple requested attributes but generally do not guarantee exact satisfaction
\citep{liu2024graphdit,jolicoeurmartineau2025anyproperty,lee2026fragfm}.
Structure-conditioned methods can provide stronger structural control, but require a user-specified scaffold or partial structure, fixing part of the molecule in advance
\citep{xu2023treeinvent,lee2025genmol}. Spectrum-conditioned methods instead require
a mass spectrum and molecular formula \citep{bohde2025diffms}. Hard-constraint
methods enforce selected rules through projection, search, or symbolic reasoning
\citep{sharma2024prodigy,madeira2024construct,saikali2026constrained,geng2026neural}.
Existing approaches therefore do not directly address our setting: satisfying multiple count constraints exactly while leaving the scaffold and partial structure unspecified.

To achieve this, we build molecules from scratch and check before each addition
whether the requested counts remain achievable. A locally valid addition can
still make these targets unreachable, yet enumerating every possible completion
would require a combinatorial search over future molecular structures. We address
both difficulties with a lossless component-tree representation. We cut each
bridge bond, whose removal disconnects the molecular graph, and treat every
intact ring system and every atom outside rings as a component. The resulting
component-level graph is a tree. Each supported count is the sum of fixed,
non-negative integer contributions from these components. Each component also has
attachment ports that identify the local atom and bond type for a connection.
Under a fixed component vocabulary and attachment rules, whether the generator
can complete a partial molecule depends only on the number of components remaining,
unmatched ports by bond type, and remaining target counts. A dynamic program can
therefore test for an exact completion using this compact state, without
enumerating complete molecular graphs.

Building on this representation, we introduce \textbf{TopoCompile}, which combines
an autoregressive Transformer for component-tree generation with a
dynamic-programming compiler.
At each step, the Transformer scores component and attachment choices, while the
compiler removes choices that leave no target-satisfying completion. The generator
then samples from the renormalized probabilities, preserving its relative
preferences among the remaining choices. Thus, the generator learns to construct
plausible molecules, while the compiler ensures that the requested counts remain
achievable. Because we apply constraints only during decoding, the same
trained generator supports different target values and count combinations
without retraining. Within the fixed component vocabulary and attachment rules,
 constrained decoding guarantees target counts for every molecule.

Our contributions are:
\begin{itemize}
  \item \textbf{Exact-control interface.}
  TopoCompile provides exact control over supported global counts while leaving
  the scaffold, component identities, and component-tree layout unspecified. The
  same trained generator supports different target values and combinations of
  supported counts without retraining.

  \item \textbf{Completion-checking compiler.}
  We derive a dynamic program over the remaining component count, unmatched ports
  by bond type, and remaining target counts. We prove that it determines exactly
  whether a target-satisfying completion exists under the fixed vocabulary and
  attachment rules, enabling exact decoding without rejection or repair.

  \item \textbf{Evaluation.}
  Across four benchmarks, TopoCompile achieves 100\% exact-valid yield for both
  frequently and infrequently observed four-count targets, whereas the strongest
  baseline in each dataset reaches at most 44.9\% on infrequent targets. In
  unconditional generation, it achieves the highest validity or yield of valid,
  unique, novel molecules among the compared methods on all four benchmarks.
\end{itemize}

\section{Related Work}
\label{sec:related}

\paragraph{Molecular generation.}
Neural molecular generators fall into two categories:
iterative global-update methods and sequential autoregressive methods.
Diffusion models
\citep{jo2022gdss,vignac2023digress,liu2024graphdit,lee2025genmol}
and discrete flow-matching models
\citep{eijkelboom2024variational,qin2025defog,lee2026fragfm}
repeatedly refine a global state represented as a molecular graph,
sequence, or fragment graph.
Autoregressive models instead construct molecules through local
decisions, adding atoms and bonds sequentially
\citep{shi2020graphaf,luo2021graphdf} or optimizing the generation
order \citep{wang2025learningorder}.
To reduce sequential complexity, several frameworks use trees to
organize motifs or guide generation
\citep{jin2018junction,jin2020hierarchical,ahn2022stgg}.
Grammar-based models constrain generation through syntactic rules
\citep{kusner2017grammar}, semantic attributes \citep{dai2018syntax},
or molecular graph grammars \citep{kajino2019mhg}.
These methods target valid construction; TopoCompile checks whether each
partial molecule can meet the requested global counts.

\paragraph{Constrained molecular generation.}
Existing methods provide different levels of control over generated molecules.
Reward-based and property-conditioned approaches bias generation toward desired
molecular properties, but generally do not guarantee exact satisfaction of
requested values \citep{you2018gcpn,liu2024graphdit,
jolicoeurmartineau2025anyproperty}.
Structure-guided methods provide stronger control by conditioning on
user-specified scaffolds or topological templates; for example, Tree-Invent
enforces user templates through masked generation \citep{xu2023treeinvent}.
However, such methods fix part of the molecular structure in advance, reducing
structural freedom.
Hard-constraint methods go further by enforcing selected constraints through
projection, constraint programming, or symbolic reasoning
\citep{sharma2024prodigy,madeira2024construct,
saikali2026constrained,geng2026neural,omrani2020constraints}.
Related mechanisms such as valence masks \citep{liu2018constrained} and
valid-by-construction representations \citep{krenn2020selfies} guarantee
aspects of chemical validity, but do not by themselves guarantee exact
satisfaction of a user-specified global count vector.
TopoCompile targets a different setting: exact control of multiple
global counts while leaving the molecular components and their arrangement
unspecified.

\section{The TopoCompile Method}
\label{sec:representation}

\noindent \textbf{Problem formulation.} We consider exact constrained molecular generation with a trained generator
$p_\theta$. Let $R:\mathcal{G}\rightarrow\mathbb{N}_0^d$ map a supported
molecular graph $G$ to a vector of discrete global counts, and let
$\mathbf{b}^\star\in\mathbb{N}_0^d$ be the user-specified target. The task is
to generate $G\in\mathcal{G}$ such that $R(G)=\mathbf{b}^\star$; if the user
additionally specifies a total component count $M$, the generated graph must
contain exactly $M$ components. The scaffold, component identities, and
component-tree layout remain unspecified, and different target values and
combinations of supported counts must be handled by the same trained generator
without retraining or post-hoc rejection or repair of completed samples.

\noindent \textbf{Overview.} Section~\ref{sec:offline-learning} introduces the
component-tree representation and autoregressive generator.
Section~\ref{sec:constraint-compilation} derives a dynamic program for exact
completion checking, and Section~\ref{sec:prefix-safe-sampling} uses it to mask
infeasible actions during decoding.
\textbf{Figure~\ref{fig:ringtree-method}} provides an overview of the complete
training and constraint-guided sampling pipeline.

\subsection{Offline learning of the component-tree generator}
\label{sec:offline-learning}
\noindent\textbf{Tree representation of molecular graphs.}
\label{sec:ringtree-representation}
We represent each molecule as a tree of molecular components. This
representation reduces autoregressive generation to adding one component to
one parent at each step while retaining enough information to reconstruct the
original molecular graph. For a connected molecular graph
$G=(V,E,\mathbf x,\mathbf y)$, an edge is a bridge if removing it disconnects
$G$. In a molecular graph, bridge edges are exactly the edges that do not
belong to any ring. Removing these non-ring edges partitions $G$ into
atom-disjoint components
$\Components(G)=\{C_1,\ldots,C_M\}$. Replacing each component with one node and
retaining one edge for each bridge produces a component-level graph $\bar T(G)$
that records the connections among components. As shown in
Figure~\ref{fig:ringtree-method}(A2), its nodes fall into two types: single
atoms and ring systems; ring-system nodes include single rings and fused,
spiro, or bridged multi-ring structures. By bridge decomposition,
$\bar T(G)$ is a tree \citep{tarjan1974bridges}.

\paragraph{Component vocabulary and encoding.}
For each component $C_i$, $\mathcal P_i$ lists the ports exposed by removed
bridges. Each port records its local atom, bridge-bond type, and symmetry class
for equivalent ports. We treat $\widehat C_i=(C_i,\mathcal P_i)$ as a
port-annotated template and write
$\widehat{\Components}(G)=\{\widehat C_i:C_i\in\Components(G)\}$.
Canonicalization independent of atom ordering gives
$\mathcal V=\operatorname{Unique}
(\bigcup_{G\in\mathcal D_{\mathrm{train}}}\widehat{\Components}(G))$.
Each component template $\tau\in\mathcal V$ stores a complete internal graph
and a fixed port set $\mathcal P(\tau)$, which determines
$\boldsymbol\sigma(\tau)$. Only canonically identical templates share a
vocabulary entry (Figure~\ref{fig:ringtree-method}(A2)).

To encode a molecule, we assign the corresponding template
$\tau_i\in\mathcal V$ to each component $C_i$, with
$\mathcal P_i=\mathcal P(\tau_i)$.
The component tree identifies which components connect, but each tree edge must
also identify the atoms that form the connection. For every edge $(i,j)$, we
record ports $p_{i\to j}$ and $p_{j\to i}$ on the two components and retain
the original bridge attribute $\mathbf y^{\mathrm{bridge}}_{ij}$. The arrows
distinguish the two ends of an undirected bond; they do not indicate bond
directions.

For the ring-system--Br connection in Figure~\ref{fig:ringtree-method}(A2),
the endpoint ports locate the attachment atoms, and
$\mathbf y^{\mathrm{bridge}}_{ij}$ records the bond type. The complete
encoding is
\begin{equation}
  \operatorname{Enc}(G)=\left(
  \{(\tau_i,\mathcal P_i)\}_{i=1}^{M},
  \{(i,j,p_{i\to j},p_{j\to i},\mathbf y^{\mathrm{bridge}}_{ij})\}_{(i,j)\in E_{\bar T}}
  \right).
  \label{eq:ringtree-encoding}
\end{equation}

\begin{figure*}[t]
  \centering
  \includegraphics[width=\textwidth]{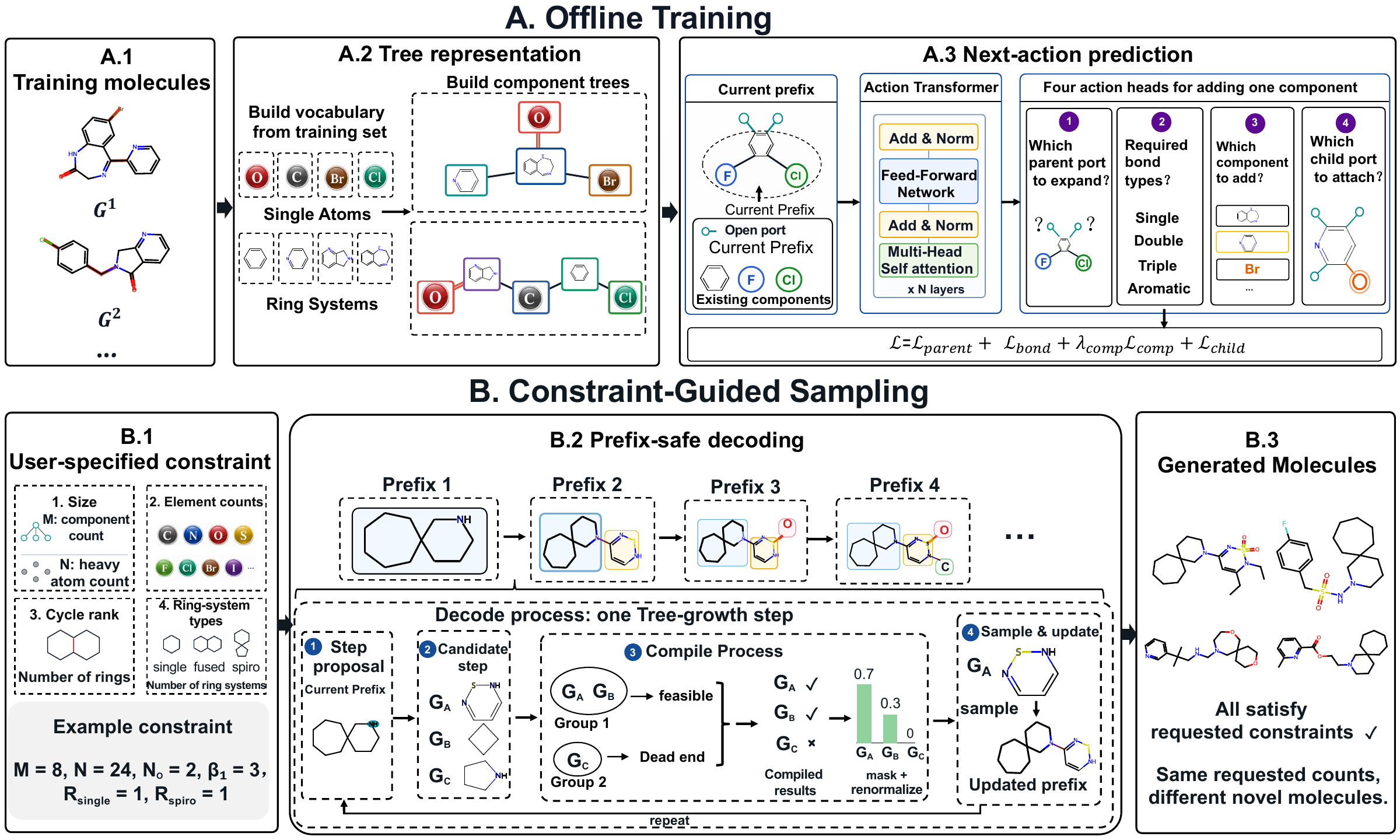}
  \caption{
  Overview of \textsc{TopoCompile}.
  (A) Molecules are encoded as trees of atom and ring-system components
  to train an autoregressive generator.
  (B) During sampling, the dynamic program masks additions with no
  target-satisfying completion; the completed tree is decoded into a
  molecule.
  }
  \label{fig:ringtree-method}
\end{figure*}

The decoder retrieves component graphs and restores bridges from their ports
and bond attributes.

Let $\mathcal G_{\mathrm{rep}}$ contain connected graphs after canonicalization
that omits stereochemical and isotopic annotations. Their atom states, bond
states, and port-local indices must lie in the declared finite encoding sets,
and their port-annotated templates must belong to $\mathcal V$.

\begin{proposition}[Lossless component-tree encoding]
For every $G\in\mathcal G_{\mathrm{rep}}$,
$\operatorname{Dec}(
\operatorname{Enc}(G))\simeq G$,
where $\simeq$ preserves the stored atom and bond attributes.
\end{proposition}

Appendix~\ref{app:inverse-proof} proves the proposition. Generation recombines
fixed component templates rather than creating new ring interiors or port
configurations; whole-molecule test coverage ranges from 93.3\% to 99.3\%
(Appendix~\ref{app:boundaries}).


\paragraph{Construction sequences.}
We randomize the root and use parent-before-child traversals to expose the
generator to multiple valid construction sequences for each molecule. The
traversal also ensures that every non-root component has an existing parent
when the model adds it. Each template $\tau$ has a bond-type vector
$\boldsymbol{\sigma}(\tau)
=(n_{\mathrm{single}},n_{\mathrm{double}},
n_{\mathrm{triple}},n_{\mathrm{aromatic}})$ that counts its attachment ports
by bridge-bond type. Let $h_t$ be the partial tree before step $t$.
A non-root action $a_t=(p_t,\boldsymbol{\sigma}_t,\tau_t,q_t)$ selects an open
parent port $p_t$, a component $\tau_t$ with
$\boldsymbol{\sigma}_t=\boldsymbol{\sigma}(\tau_t)$,
and a child port $q_t$ of matching bond type.
Adding the component joins the selected ports with a bridge, consumes
both ports, and leaves its other ports open.
The root action selects only a bond-type vector and component, with
$p_1=q_1=\mathrm{null}$.

\paragraph{Next-action prediction and training.}
We train an autoregressive Transformer for
$p_\theta(a_t\mid h_t,M)$, where $M$ is the total component count.
Four prediction heads factor each addition, allowing compatibility to narrow
successive choices without a separate class for every parent-port, bond-type vector,
component, and child-port tuple.
The parent-port and bond-vector heads predict where to extend the tree
and the new component's connection requirements.
The component head selects from
$\mathcal V(\boldsymbol{\sigma}_t)
=\{\tau\in\mathcal V:
\boldsymbol{\sigma}(\tau)=\boldsymbol{\sigma}_t\}$,
and the child-port head selects a port matching the parent bond type.
Under teacher forcing, each head predicts its corresponding field of
$a_t^\star$ from the ground-truth partial tree $h_t^\star$, using the
weighted cross-entropy loss
$\mathcal L=\mathcal L_{\mathrm{bond}}
+\lambda_{\mathrm{comp}}\mathcal L_{\mathrm{comp}}
+\mathcal L_{\mathrm{parent}}+\mathcal L_{\mathrm{child}}$,
where $\lambda_{\mathrm{comp}}>0$ weights component prediction
(Appendix~\ref{app:model-training}).
At the root, only bond-vector and component losses apply because no parent
attachment exists.
Unconditional generation draws $M$ from its empirical training distribution
and samples actions autoregressively from $p_\theta$;
constrained generation reuses this model with the completion checks.

\subsection{Compiling exact global constraints}
\label{sec:constraint-compilation}
\label{sec:additive-resources}
\label{sec:structural-reachability}
\label{sec:resource-compiler}

The learned generator scores plausible component additions, whereas exact
control requires knowing whether each addition leaves a target-satisfying
completion. Model probabilities alone do not provide this feasibility test.
A dynamic program performs the test from a compact summary of the remaining
requirements without enumerating future molecular structures.

\noindent\textbf{Targets and component contributions.}
The user specifies a vector of target counts $\mathbf b^\star$, without
specifying components or their tree layout. For a fixed total component count,
we write
$\mathbf c^\star=(M,\mathbf b^\star)$.
When the user omits $M$, we check
$M=1,\ldots,M_{\max}$ and sample among the feasible values using the
restricted empirical prior in Appendix~\ref{app:evaluation-protocols}.
Here, $M_{\max}$ is the generator's component-count limit.

We call $r:\mathcal V\to\mathbb N_0^d$ a supported count map for $R$ if it
assigns each template a non-negative integer contribution, attachment leaves these
contributions unchanged, and every constructible molecule satisfies
$R(G)=\sum_{i=1}^{M}r(\tau_i)$.
These conditions allow us to update the remaining targets by subtracting
$r(\tau)$ after adding a component.
Heavy-atom and element counts satisfy them directly. Ring-system count
and cycle rank do as well because bridge attachments do not create
cycles between components. Appendix~\ref{app:cycle-additivity} derives this
additivity.
The current compiler supports exact equalities over these non-negative,
component-additive integer counts. Quantities that depend on component
connections, such as cross-component motifs, require a different state or
acceptance rule. The formulation excludes continuous descriptors, inequalities,
and optimization objectives.

\noindent\textbf{A compact state for completion checking.}
A partial tree's completion feasibility depends on its remaining component
budget, open attachment ports, and residual target counts. After root
selection, we collect them in $z=(k,\mathbf f,\mathbf b)$: $k$ counts the
components still to add, $\mathbf f\in\mathbb N_0^4$ counts open ports by bond
type, and $\mathbf b$ contains the counts that future components must supply.
The coordinates of $\mathbf f$ match $\boldsymbol\sigma(\tau)$ but summarize
the entire partial tree.

The compiler treats two ports as compatible exactly when their bridge-bond types match.
Atom locations change the molecule but not compatibility, and each template fixes
its count contribution. Partial trees with the same
$(k,\mathbf f,\mathbf b)$ share the same completion-feasibility result.

\paragraph{Grouping additions with the same effect.}
For an action $a=(p,\boldsymbol{\sigma},\tau,q)$, the state update depends only on
$\gamma(a)=(\boldsymbol{\sigma}(\tau),r(\tau),\ell(q))$, where $\ell(q)$ is the
child port's bond type.
We group actions by this tuple and denote the finite group set by
$\Gamma$.
For $\gamma=(\boldsymbol{\sigma}_\gamma,r_\gamma,\ell_\gamma)$,
let $\mathbf e_{\ell_\gamma}$ denote the unit vector for bond type
$\ell_\gamma$. The transition is
\begin{equation}
\operatorname{Next}(z,\gamma)
=\bigl(k-1,\ \mathbf f+\boldsymbol{\sigma}_\gamma-2\mathbf e_{\ell_\gamma},\
\mathbf b-r_\gamma\bigr).
\label{eq:tc-compact-next}
\end{equation}
The new component supplies $r_\gamma$ and its ports; the bridge
consumes one existing and one new open port of type $\ell_\gamma$.
Grouping shares only feasibility computation: the generator still
scores concrete components and attachment sites.

\paragraph{Joint feasibility DP.}
A feasible continuation must use one sequence of additions to close every
open port and reduce every residual count to zero.
Let $\Gamma(z)$ contain groups satisfying
$\mathbf f_{\ell_\gamma}\geq1$,
$(\boldsymbol{\sigma}_\gamma)_{\ell_\gamma}\geq1$,
and $r_\gamma\leq\mathbf b$ coordinatewise.
Define $F_{\mathbf c^\star}(z)$ to indicate whether exactly $k$ further
additions can close all ports and supply the remaining counts:
\begin{equation}
F_{\mathbf c^\star}(0,\mathbf f,\mathbf b)
=\mathbf 1[\mathbf f=\mathbf 0\land\mathbf b=\mathbf 0],\qquad
F_{\mathbf c^\star}(z)=\bigvee_{\gamma\in\Gamma(z)}
F_{\mathbf c^\star}\!\left(\operatorname{Next}(z,\gamma)\right),\quad k>0.
\label{eq:tc-compact-feasibility}
\end{equation}
For example, with one component left, the same template must close all open
ports and supply the residual counts; different templates satisfying these
requirements separately do not suffice.
Separate structural and count reachability checks may succeed through
different continuations, so the compiler uses them only to prune states.
It evaluates the joint recurrence on demand and memoizes results across
construction histories. Appendix~\ref{app:full-compiler} gives the algorithm
and pruning rules.

\subsection{Constrained sampling}
\label{sec:prefix-safe-sampling}
\label{sec:compiler-guided-sampling}

\paragraph{Constrained decoding.}
For fixed $\mathbf c^\star=(M,\mathbf b^\star)$, a root component $\tau$ with
$r(\tau)\leq\mathbf b^\star$ leaves state
$z_{\mathrm{root}}(\tau)
=(M-1,\boldsymbol{\sigma}(\tau),\mathbf b^\star-r(\tau))$.
We retain roots satisfying
$F_{\mathbf c^\star}(z_{\mathrm{root}}(\tau))=1$
and sample from their renormalized model probabilities.
When the user does not specify $M$, these checks test all allowed values before
selecting among feasible roots. At each non-root step,
the partial component tree $h_t$ provides the
generator's molecular context, while $z_t$ specifies the completion
requirements.
For each queried group, we cache
$\Lambda_{\mathbf c^\star}(z,\gamma)
=F_{\mathbf c^\star}(\operatorname{Next}(z,\gamma))$.
Provided the retained choices have positive total model probability,
we sample from
\begin{equation}
  \widetilde p_\theta(a\mid h_t,z_t,\mathbf c^\star)
  =
  \frac{
    p_\theta(a\mid h_t,M)\Lambda_{\mathbf c^\star}(z_t,\gamma(a))
  }{
    \sum_{a'}
    p_\theta(a'\mid h_t,M)\Lambda_{\mathbf c^\star}(z_t,\gamma(a'))
  }.
  \label{eq:tc-compact-sampling}
\end{equation}
This preserves probability ratios among retained actions.
It does not in general equal the original model distribution conditioned on
the targets because Boolean feasibility records the existence of a matching
continuation, not its model probability.
Thus, exactness here refers to satisfaction of the requested counts, not exact
sampling from $p_\theta(\cdot\mid R=\mathbf b^\star)$.
After sampling, we attach the component, update $z_t$ using
Equation~\eqref{eq:tc-compact-next}, and continue until $k=0$.
We then decode the completed component tree into a molecular graph.

\begin{theorem}[Prefix safety and exact completion]
Fix a finite vocabulary $\mathcal V$, the bond-type matching attachment rules
defined above,
a supported count map $r:\mathcal V\to\mathbb N_0^d$, and a target
$\mathbf c^\star=(M,\mathbf b^\star)$.
Assume that the feasible roots have positive total model probability and that
the retained actions do so at every nonterminal state reached during sampling.
Then every sampled partial tree admits a target-satisfying continuation, and
sampling can continue until $k=0$. Every terminal sample has exactly $M$
components, all ports matched, and $R(G)=\mathbf b^\star$.
\end{theorem}

\textit{Proof sketch.}
The root check starts from a feasible state, and each retained action preserves
completion feasibility. At $k=0$, the terminal condition
$\mathbf f=\mathbf b=\mathbf 0$ gives port closure and exact target counts.
Appendix~\ref{app:resource-proof} gives the full proof.

\paragraph{Scope and failure behavior.}
The guarantee is relative to the component vocabulary and the bond-type
matching rules above. The compiler certifies infeasibility only if every
allowed $M$ compiles and none admits a feasible root. A joint-DP state-limit
error makes no infeasibility claim. Zero model probability across retained
choices leaves the sampling distribution undefined despite DP feasibility.
Exact count satisfaction and port closure do not establish chemical validity
or synthesizability; Section~\ref{sec:experiments} evaluates validity separately.

\paragraph{Time complexity.}
Let $d$ count controlled quantities, $S$ all DP states evaluated during
compilation or sampling, and $G_{\max}$ the maximum groups examined per state.
$T_{\mathrm{dec}}(M)$ covers model
prediction, candidate processing, cached checks, and reconstruction. After
preprocessing, compiling one target and generating $B$ molecules costs
$O(dSG_{\max}+B\,T_{\mathrm{dec}}(M))$. All outputs share the compiled table,
giving an amortized compilation cost of
$O(dSG_{\max}/B)$ per molecule. Its pseudo-polynomial dependence on target
values can make large joint requests expensive. Appendix~\ref{app:time-complexity}
derives the bound and covers the case where the user does not specify $M$.

\section{Experiments}
\label{sec:experiments}
\raggedbottom

We evaluate whether \textsc{TopoCompile} delivers exact counts while preserving
generation quality and practical cost. We study four questions: \textbf{RQ1:} How reliably does
\textsc{TopoCompile} generate molecules with requested counts? \textbf{RQ2:}
What do joint completion checking and learned preferences contribute?
\textbf{RQ3:} How competitive is \textsc{TopoCompile} in unconditional
generation? \textbf{RQ4:} How do methods compare in exact-output delivery
efficiency?
\subsection{Experimental Setup}

\paragraph{Datasets and tasks.}
We evaluate on MOSES \citep{polykovskiy2020moses}, ZINC250K
\citep{sterling2015zinc}, GuacaMol \citep{brown2019guacamol}, and NPGen
\citep{lee2026fragfm}. TopoCompile supports different combinations of
component-additive count constraints. We evaluate eight constraint schemas over
heavy-atom count $N$, cycle rank $\beta$, oxygen count $N_{\rm O}$, and
ring-system count $R_{\mathrm{sys}}$. A four-count request specifies target
values for all four quantities $(N,\beta,N_{\rm O},R_{\mathrm{sys}})$, whereas
the other schemas specify only one, two, or three of them. The user does not
specify the component count $M$; TopoCompile instead considers supported $M$
values and selects among those admitting a feasible completion. The main text
focuses on the four-count schema, while results for the subset schemas are
reported in Appendix~\ref{app:shared-control}. For unconditional generation,
where no count targets are given, TopoCompile samples $M$ directly from its
empirical training distribution and generates autoregressively. Detailed task
settings and evaluation protocols are provided in Appendices~\ref{app:evaluation-protocols}
and~\ref{app:shared-control}.

\paragraph{Baselines.}
For exact-control generation, we compare \textsc{TopoCompile} with Graph DiT
\citep{liu2024graphdit}, STGG+ \citep{jolicoeurmartineau2025anyproperty},
FragFM+Guidance \citep{lee2026fragfm}, GenMol+Filter \citep{lee2025genmol},
and STGG+ Count Mask under the same target requests and attempt budgets.
Appendix~\ref{app:baseline-implementations} details baseline implementations
and configurations. For unconditional generation, Table~\ref{tab:standard-compact}
lists the established molecular generators compared with \textsc{TopoCompile}.

\paragraph{Evaluation metrics.}
For exact-control generation, \textbf{Valid} measures chemical validity, while
\textbf{Exact-valid} additionally requires every requested count to be satisfied
exactly. \textbf{Novelty} and \textbf{internal diversity} assess whether exact
outputs are new and structurally diverse. In the sampling-policy ablation,
where all compared policies use the same exact completion check, conditional FCD
measures how closely their exact outputs match the target-specific reference
distribution. For unconditional generation, we follow each benchmark's
standard Validity, valid-unique-novel (V.U.N.) yield, and FCD protocol.
Appendices~\ref{app:evaluation-protocols} and~\ref{app:shared-control} detail
the settings.

\subsection{RQ1: How reliably does TopoCompile meet requested count constraints?}
\label{sec:shared-control-results}

\begin{table}[H]
\centering
\caption{Exact four-count generation
$(N,\beta,N_{\rm O},R_{\mathrm{sys}})$. \textbf{(A)} Exact-valid yield on
Common and Rare targets. \textbf{(B)} Novelty and internal diversity of matched
Common exact outputs.}
\label{tab:shared-exact-control}
\scriptsize
\setlength{\tabcolsep}{0.7pt}
\renewcommand{\arraystretch}{1.04}

\begin{tabularx}{\textwidth}{@{}l!{\vrule width 0.6pt}YY!{\vrule width 0.6pt}YY!{\vrule width 0.6pt}YY!{\vrule width 0.6pt}YY@{}}
\toprule
\multicolumn{9}{@{}l}{\textit{A. Exact-valid yield (\%)}} \\
\rowcolor{tablegray}
\textbf{Method}
& \multicolumn{2}{c!{\vrule width 0.6pt}}{\textbf{MOSES}}
& \multicolumn{2}{c!{\vrule width 0.6pt}}{\textbf{ZINC250K}}
& \multicolumn{2}{c!{\vrule width 0.6pt}}{\textbf{GuacaMol}}
& \multicolumn{2}{c}{\textbf{NPGen}} \\
\rowcolor{tablegray}
& \textit{Common} & \textit{Rare}
& \textit{Common} & \textit{Rare}
& \textit{Common} & \textit{Rare}
& \textit{Common} & \textit{Rare} \\
\midrule
\rowcolor{tablegray}
Graph DiT & 56.64 & 5.57 & 43.36 & 8.40 & 0.49 & 4.30 & \multicolumn{2}{c}{\textsc{OOM}} \\
STGG+ & 80.86 & 29.00 & 79.59 & 37.50 & 34.86 & 13.87 & 16.89 & 3.71 \\
\rowcolor{tablegray}
FragFM+Guidance & 22.07 & 0.20 & 11.04 & 0.00 & 4.98 & 0.00 & 1.56 & 0.00 \\
GenMol+Filter & 1.86 & 0.00 & 1.46 & 0.00 & 1.66 & 0.00 & 0.00 & 0.00 \\
\rowcolor{tablegray}
STGG+ Count Mask & 83.59 & 35.55 & 82.71 & 44.92 & 46.88 & 15.53 & 27.44 & 14.06 \\
\rowcolor{cyan!15}
\textbf{TopoCompile (ours)} & \textbf{100.00} & \textbf{100.00} & \textbf{100.00} & \textbf{100.00} & \textbf{100.00} & \textbf{100.00} & \textbf{100.00} & \textbf{100.00} \\
\bottomrule
\end{tabularx}

\vspace{1.5pt}

\begin{tabularx}{\textwidth}{@{}l!{\vrule width 0.6pt}YY!{\vrule width 0.6pt}YY!{\vrule width 0.6pt}YY!{\vrule width 0.6pt}YY@{}}
\toprule
\multicolumn{9}{@{}l}{\textit{B. Matched exact-output novelty and diversity}} \\
\rowcolor{tablegray}
\textbf{Method}
& \multicolumn{2}{c!{\vrule width 0.6pt}}{\textbf{MOSES}}
& \multicolumn{2}{c!{\vrule width 0.6pt}}{\textbf{ZINC250K}}
& \multicolumn{2}{c!{\vrule width 0.6pt}}{\textbf{GuacaMol}}
& \multicolumn{2}{c}{\textbf{NPGen}} \\
\rowcolor{tablegray}
& Novelty $\uparrow$ & Div. $\uparrow$
& Novelty $\uparrow$ & Div. $\uparrow$
& Novelty $\uparrow$ & Div. $\uparrow$
& Novelty $\uparrow$ & Div. $\uparrow$ \\
\midrule
\rowcolor{tablegray}
STGG+ & 0.946 & 0.845 & \textbf{1.000} & 0.860 & 0.994 & 0.867 & \textbf{0.978} & 0.883 \\
FragFM+Guidance & 0.899 & 0.846 & 0.987 & 0.852 & 0.988 & 0.858 & 0.689 & 0.860 \\
\rowcolor{tablegray}
STGG+ Count Mask & 0.952 & 0.847 & \textbf{1.000} & 0.861 & 0.994 & 0.869 & \textbf{0.978} & 0.883 \\
\rowcolor{cyan!15}
\textbf{TopoCompile (ours)} & \textbf{0.990} & \textbf{0.862} & \textbf{1.000} & \textbf{0.872} & \textbf{1.000} & \textbf{0.882} & 0.956 & \textbf{0.900} \\
\bottomrule
\end{tabularx}
\end{table}

\begin{figure*}[t]
\centering
\includegraphics[width=\textwidth]{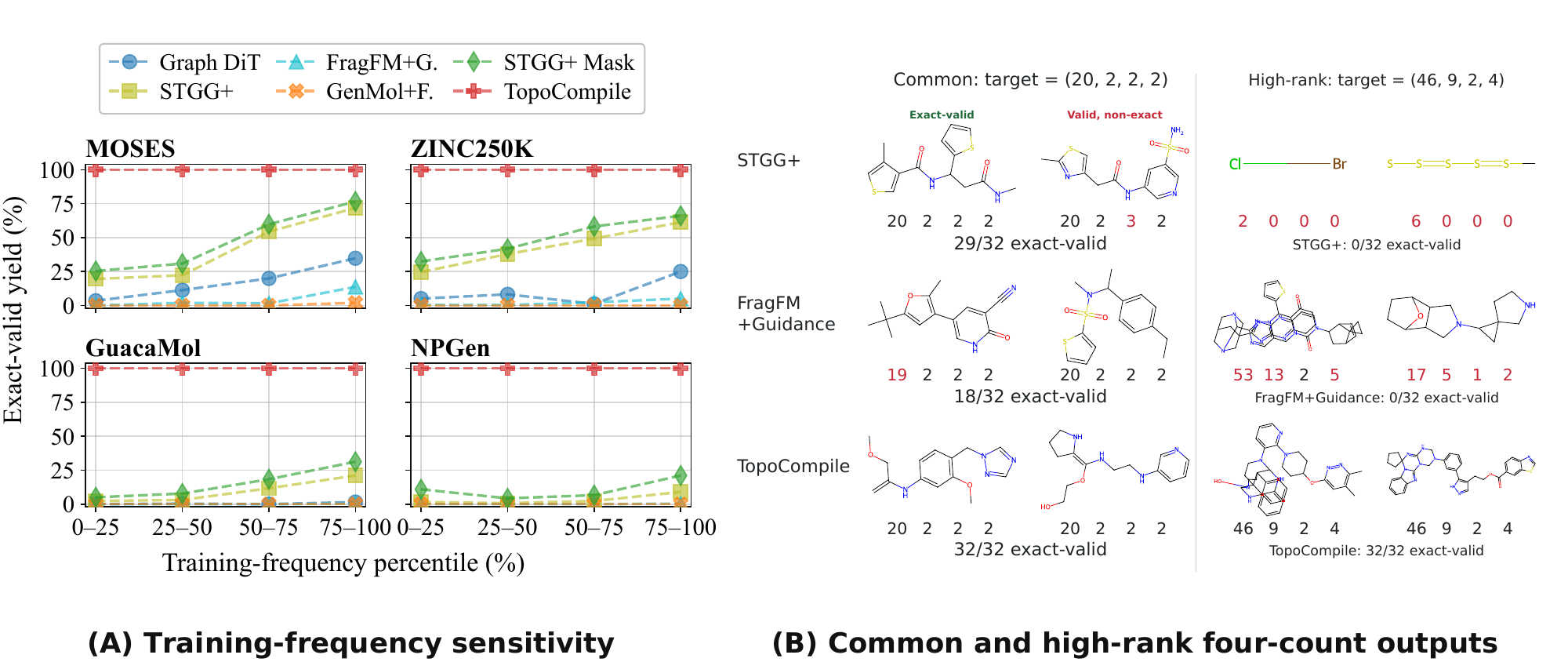}
\caption{Exact control across target frequencies and representative
outputs. (A) Exact-valid yield across training-frequency quartiles. (B)
Representative Common and high-cycle-rank outputs.}
\label{fig:frequency-response}
\end{figure*}

\paragraph{Overall performance.}
For each dataset, observed four-count tuples with $R_{\mathrm{sys}}\ge1$ are
\textbf{Common} at or above the 75th percentile of training frequency and
\textbf{Rare} at or below the 25th. We select up to 32 targets per group.
Table~\ref{tab:shared-exact-control}(A) reports Exact-valid yield over 32
attempts per target. Table~\ref{tab:shared-exact-control}(B) compares novelty
and internal diversity using equal numbers of exact-valid molecules per Common
target from each reported method's first 100 attempts; targets with fewer than
two from any method are excluded (Appendix~\ref{app:evaluation-protocols}).
Figure~\ref{fig:frequency-response}(A) ranks tuples by frequency, divides them
into quartiles, and evaluates eight targets per quartile;
Figure~\ref{fig:frequency-response}(B) shows Common and high-cycle-rank
outputs.

\textbf{As shown in Table~\ref{tab:shared-exact-control},}
\textsc{TopoCompile} achieves 100\% Exact-valid yield on both Common and Rare
four-count targets across all four datasets, whereas the strongest baseline in
each dataset reaches only 14.1--44.9\% on Rare targets. On matched Common
exact outputs, \textsc{TopoCompile} also maintains high novelty and achieves the
highest internal diversity on all four datasets. \textbf{Figure~\ref{fig:frequency-response}(A)}
shows that baseline Exact-valid yield generally increases with frequency, while
\textsc{TopoCompile} remains at 100\% across all four quartiles.
\textbf{Figure~\ref{fig:frequency-response}(B)} shows that competing methods
can produce valid molecules while missing target counts, whereas \textsc{TopoCompile}
satisfies all four, including a high-cycle-rank request beyond the MOSES
training maximum.

\subsection{RQ2: What do joint checks and learned preferences contribute?}
\label{sec:resource-results}

We conduct two ablations on MOSES and ZINC250K across five constraint schemas,
pooling Common and Rare requests. First, with the generator and sampling policy
fixed, we vary the completion check: Local rejects immediate structural or
count violations; Factorized additionally checks structural and count
completion separately, potentially using different continuations; Joint
requires one continuation satisfying both. Second, with Joint fixed, Uniform
assigns equal probabilities to legal options for each action field, Frequency
uses their training frequencies, and Learned uses Transformer probabilities.
Figure~\ref{fig:rq2-ablation}(a,b) reports Exact-valid yield, and
Figure~\ref{fig:rq2-ablation}(c,d) reports conditional FCD.
Appendix~\ref{app:resource-experiments} provides details.

As shown in Figure~\ref{fig:rq2-ablation}(a,b), Joint achieves 100\%
Exact-valid yield across all evaluated settings, versus 67.1--86.2\% for
Factorized and 12.4--56.2\% for Local, showing that separate feasibility checks
cannot replace joint completion checking. Under Joint, all three sampling
policies retain 100\% Exact-valid yield, but Learned achieves the lowest
conditional FCD in every schema (Figure~\ref{fig:rq2-ablation}(c,d)). Learned
therefore improves alignment with target-specific references, although Uniform
has higher internal diversity in some settings (Appendix~\ref{app:resource-experiments}).

\begingroup
\setlength{\intextsep}{0pt}
\begin{figure}[H]
\centering
\includegraphics[width=0.90\columnwidth,trim=0 8bp 0 8bp,clip]{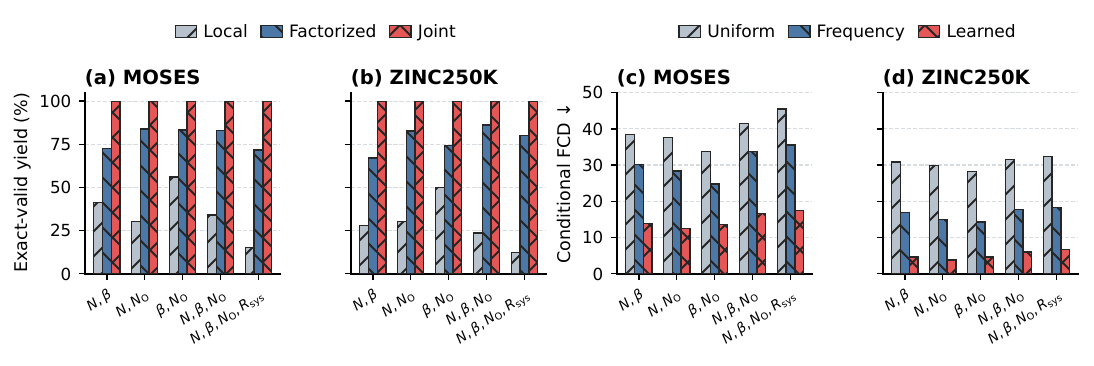}
\caption{Schema-level ablations on MOSES and ZINC250K with Common and Rare
requests pooled. \textbf{(a,b)} Exact-valid yield under Local, Factorized, and
Joint completion checks. \textbf{(c,d)} Conditional FCD under Uniform,
Frequency, and Learned sampling with Joint checking fixed.}
\label{fig:rq2-ablation}
\end{figure}
\endgroup

\subsection{RQ3: How competitive is TopoCompile in unconditional generation?}
\label{sec:controls}

\begin{table}[!t]
\centering
\caption{Unconditional molecular generation. GuacaMol reports FCD score
($\uparrow$); other datasets report FCD distance ($\downarrow$). Bold indicates
the best result in each column.}
\label{tab:standard-compact}
\scriptsize
\setlength{\tabcolsep}{2.0pt}
\renewcommand{\arraystretch}{1.06}
\begin{tabularx}{\textwidth}{@{}l!{\vrule width 0.6pt}YY!{\vrule width 0.6pt}YY!{\vrule width 0.6pt}YY!{\vrule width 0.6pt}YY@{}}
\toprule
\rowcolor{tablegray}
\textbf{Method}
& \multicolumn{2}{c!{\vrule width 0.6pt}}{\textbf{GuacaMol}}
& \multicolumn{2}{c!{\vrule width 0.6pt}}{\textbf{MOSES}}
& \multicolumn{2}{c!{\vrule width 0.6pt}}{\textbf{NPGen}}
& \multicolumn{2}{c}{\textbf{ZINC250K}} \\
\rowcolor{tablegray}
& V.U.N. $\uparrow$ & \mbox{\fontsize{5.5}{5.5}\selectfont\textbf{FCD score $\uparrow$}}
& V.U.N. $\uparrow$ & FCD $\downarrow$
& V.U.N. $\uparrow$ & FCD $\downarrow$
& Validity $\uparrow$ & FCD $\downarrow$ \\
\midrule
\rowcolor{tablegray}
DiGress \citep{vignac2023digress}
& 85.1 & 68.0 & 81.4 & 1.19 & 85.1 & 2.05 & 94.98 & 3.482 \\
DeFoG \citep{qin2025defog}
& 97.9 & 73.8 & 85.4 & 1.95 & 83.8 & 4.46 & 99.22 & 1.425 \\
\rowcolor{tablegray}
GraphAF \citep{shi2020graphaf,lee2026fragfm}
& \multicolumn{2}{c!{\vrule width 0.6pt}}{--}
& 85.2 & 0.523
& 48.1 & 25.11 & 67.92 & 16.128 \\
Tree-Invent \citep{xu2023treeinvent}
& 86.2 & 58.1 & \multicolumn{2}{c!{\vrule width 0.6pt}}{OOM}
& 75.3 & 8.20 & 95.43 & 5.999 \\
\rowcolor{tablegray}
FragFM \citep{lee2026fragfm}
& 95.0 & \textbf{85.8} & 86.9 & \textbf{0.58}
& 92.6 & 1.34 & 99.81 & 0.630 \\
\rowcolor{cyan!15}
\textbf{TopoCompile (ours)}
& \textbf{98.2} & 80.8 & \textbf{92.5} & 0.88
& \textbf{97.4} & \textbf{0.841} & \textbf{100.00} & \textbf{0.520} \\
\bottomrule
\end{tabularx}
\end{table}

As shown in Table~\ref{tab:standard-compact}, \textsc{TopoCompile} achieves the
highest V.U.N. yield on GuacaMol, MOSES, and NPGen, and the highest validity on
ZINC250K. On MOSES and NPGen, it improves V.U.N. over the strongest reported
baseline by 5.6 and 4.8 percentage points, respectively. Distributional
alignment shows a more mixed picture: \textsc{TopoCompile} achieves the best
reported FCD on NPGen and ZINC250K, while FragFM performs better on GuacaMol and
MOSES.
These results also separate the two roles in \textsc{TopoCompile}: the
component-tree generator remains competitive without constraint checks, while
the compiler is responsible for exact control.

\subsection{RQ4: How do methods compare in exact-output delivery efficiency?}
\label{sec:efficiency-results}

Figure~\ref{fig:moses-generation-efficiency}(a) measures the number of
exact-valid molecules generated within a fixed five-minute budget.
Figure~\ref{fig:moses-generation-efficiency}(b) compares the cost of adapting
to a new target and the resulting Exact-valid yield.
Figure~\ref{fig:moses-generation-efficiency}(c) measures the benefit of reusing
cached compilation results, and Figure~\ref{fig:moses-generation-efficiency}(d)
examines how generating multiple molecules for the same target amortizes the
one-time compilation cost.

As shown in Figure~\ref{fig:moses-generation-efficiency}, \textsc{TopoCompile}
generates more exact-valid molecules within the fixed time budget, demonstrating
higher end-to-end efficiency (a). For new targets, its main advantage is that
adaptation only requires compilation, with no retraining or tuning of the
generator, while filtering avoids adaptation at the cost of much lower
exact-valid yield (b). Cache reuse further reduces latency for repeated
requests by reusing previously solved feasibility states (c). Finally, more
molecules amortize a target's one-time compilation cost, but one-off requests
pay the full cost (d). A fixed-$M$ ablation shows that grouping equivalent
actions reduces median compilation time by 12.2--13.9$\times$ without changing
the DP state count (Appendix~\ref{app:resource-experiments}).

\begin{figure}[H]
\centering
\includegraphics[width=\textwidth]{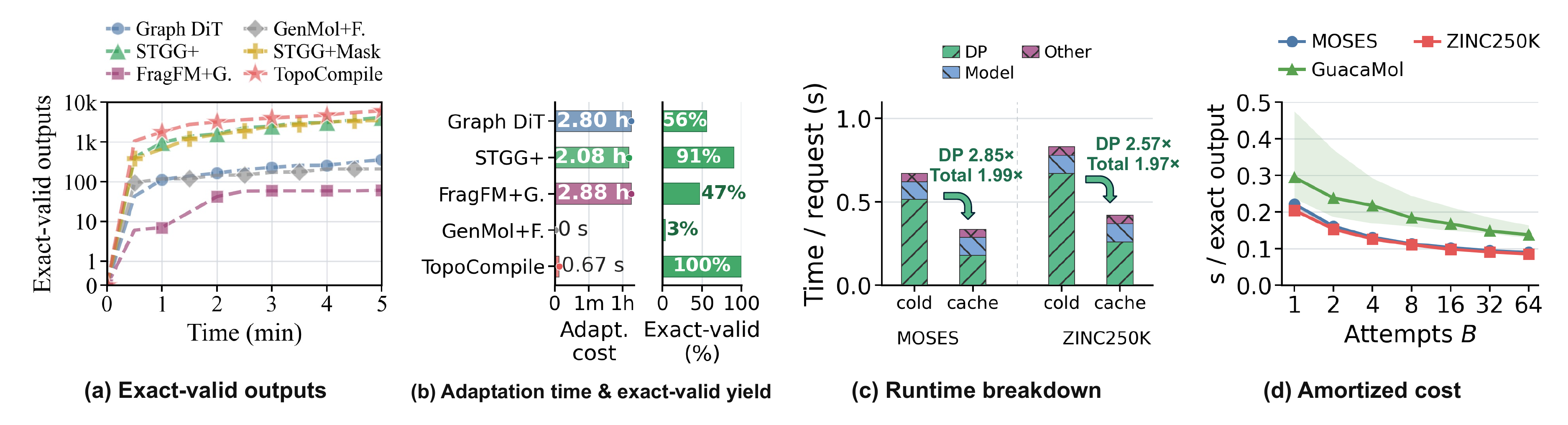}
\caption{Efficiency of exact constrained generation. (a) Exact-valid molecules
generated within a fixed time budget. (b) Target adaptation cost and Exact-valid
yield. (c) Cold and cached request time. (d) Amortized time per exact-valid
molecule as the number of samples per target increases.}
\label{fig:moses-generation-efficiency}
\end{figure}

\section{Conclusion}
\label{sec:conclusion}

We introduce TopoCompile for molecular generation under exact global count constraints. Its autoregressive component-tree generator ranks plausible additions, while dynamic-programming completion checks retain only additions that can reach the requested counts. TopoCompile satisfies supported counts exactly while retaining molecular quality and diversity, and its trained generator also produces high-quality molecules without constraints. Future work will extend exact control to connection-dependent constraints through richer completion states and broaden component support to enlarge the reachable molecular space while preserving the current exact completion guarantees.

\clearpage
\subsection*{Reproducibility Statement}

Appendix~\ref{app:additional-settings}
documents the data preprocessing, model configuration, and evaluation protocols needed
to reproduce the reported results.

\subsection*{Ethics Statement}

This work is computational and does not involve human subjects or private data.
We acknowledge that molecular generation may have downstream biomedical
applications and should be used with appropriate experimental validation.

\subsection*{AI use statement}

Generative AI tools were used only to assist with writing and language polishing,
including improving clarity, grammar, and presentation. All technical content,
methods, experiments, results, analyses, and conclusions were developed and
verified by the authors. The authors take responsibility for the final code,
experiments, claims, citations, figures, and text.

\setlength{\bibsep}{2pt plus 1pt minus 0.5pt}
\bibliography{references}
\bibliographystyle{iclr2027_conference}

\clearpage
\appendix
\setlength{\textfloatsep}{8pt plus 2pt minus 2pt}
\setlength{\floatsep}{8pt plus 2pt minus 2pt}
\setlength{\intextsep}{8pt plus 2pt minus 2pt}
\section{Proofs and Implementation Details}
\label{app:proofs}

\subsection{Bridge-component tree}
\label{app:bridge-tree-proof}

\begin{proof}
Let $G=(V,E)$ be connected, let $\Bridges$ be its bridges, and let
$C_1,\ldots,C_M$ be the connected components of
$(V,E\setminus\Bridges)$. Construct $\bar T(G)$ with one vertex per component
and one edge per bridge. Contracting connected subgraphs preserves
connectivity, so $\bar T(G)$ is connected. Suppose $\bar T(G)$ contained a
cycle and choose an edge on that cycle corresponding to bridge $(u,v)$. The
other cycle edges correspond to a sequence of components connected by bridges
without using $(u,v)$. Paths inside these components then connect $u$ to $v$
without that edge, contradicting that $(u,v)$ is a bridge. Hence $\bar T(G)$ is
acyclic and connected, and therefore is a tree with $M-1$ edges.
\end{proof}

\subsection{Exact inverse}
\label{app:inverse-proof}

\begin{proof}
The components $C_i$ partition $V$. For each $C_i$, template $\tau_i$ stores
one atom state for every $v\in C_i$ and one bond state for every edge whose
endpoints both lie in $C_i$. Every remaining edge is a bridge. Its tree-edge
record stores the two component indices, one port slot per endpoint, and the
bridge-bond state; each port slot stores the component-local endpoint atom.
The decoder first creates disjoint copies of all template atoms and internal
bonds, then adds exactly one bridge for each tree-edge record between the atoms
named by its ports. Thus every source vertex and edge is created exactly once,
and no additional vertex or edge is introduced. The source-to-decoded local
index maps define a feature-preserving graph isomorphism.

This argument applies to connected inputs after non-isomeric canonicalization
when all atom states, bond states, and port-local indices belong to the declared
representation. Stereochemical and isotopic identity is not stored. The
implementation uses RDKit 2025.09.2; different RDKit versions can assign
different component identities and are therefore outside the reproducibility
claim.
\end{proof}

\subsection{Time complexity and computational efficiency}
\label{app:time-complexity}

We analyze inference for a fixed component vocabulary and resource map.
Integer coordinates are assumed to fit in machine words, and table access
is constant-time after constructing the key. The cost of constructing
a state key is included in the vector-processing cost below.

\paragraph{DP compilation.}
Let $S$ count the structurally admissible DP states evaluated, including
one additional state for root selection, and let $G_{\max}\geq 1$ bound
the number of groups examined at any state. States rejected by an immediate
structural check are charged to the transition that produces them.

Each transition updates four port counts and $d$ resource counts.
For $d\geq 1$, this takes $O(d)$ time, including compatibility checks
and construction of the successor key. Memoization evaluates each state
once, and each evaluation examines at most $G_{\max}$ groups.
The total DP time is therefore
\[
  T_{\rm DP}=O(dSG_{\max}).
\]
This bound includes states first evaluated during sampling if compilation
is performed lazily. Such evaluations contribute to $S$ rather than being
charged as a new search at every generation step.

To bound the state count, consider a partial tree with $k$ components
left to add. Each currently open port must eventually connect to a new
component. Since one addition connects a new component to only one existing
port, a structurally feasible state must satisfy
\[
  \sum_{\ell=1}^{4}f_\ell\leq k.
\]
There are $\binom{k+4}{4}$ non-negative four-dimensional vectors satisfying
this inequality. Each residual resource coordinate also satisfies
$0\leq b_j\leq b_j^\star$. Consequently,
\begin{align}
  S
  &\leq
  1+\sum_{k=0}^{M-1}
  \binom{k+4}{4}
  \prod_{j=1}^{d}(b_j^\star+1) \nonumber\\
  &=
  1+\binom{M+4}{5}
  \prod_{j=1}^{d}(b_j^\star+1).
  \label{eq:dp-state-bound}
\end{align}
The leading one accounts for root selection.
Thus, for fixed $d$, the bound is polynomial in $M$ and the numerical
target values, but not in their binary encoding lengths.
The product over target coordinates also explains why increasing the
number of jointly controlled counts can be expensive.

This is a worst-case bound, not the number of states necessarily visited.
The actual computation follows reachable transitions and uses the pruning
rules in Appendix~\ref{app:full-compiler}.
Grouping reduces repeated work across actions with the same effect,
while memoization reduces repeated work across construction histories
that reach the same state. Neither guarantees a small state space for
every target.

\paragraph{Generation with cached completion checks.}
At generation step $t$, let $C_\theta(t)$ be the cost of computing the
candidate probabilities used by the sampler, including the prediction
heads. Let $A_t$ be the number of concrete candidate scores processed
at that step.

Once the feasibility results are cached and indexed, checking one candidate
takes constant time. Constructing the mask, normalizing the retained scores,
and sampling therefore take $O(A_t)$ time in total, not $O(1)$.
Updating the selected DP state adds $O(d)$ work.
After $M$ additions, reconstructing the atom--bond graph takes
$O(|V|+|E|)$ time. Hence, the decoding cost is bounded by
\[
  T_{\rm dec}(M)
  =
  O\!\left(
    \sum_{t=1}^{M}\bigl(C_\theta(t)+A_t+d\bigr)
    +|V|+|E|
  \right).
  \label{eq:decoding-time}
\]
The bound covers graph construction, not chemical sanitization or
benchmark evaluation.

For a fixed vocabulary, component sizes are bounded, so
$|V|+|E|=O(M)$. For fixed $d$, state updates also take $O(M)$ time
across the trajectory. These terms can therefore be absorbed into the
candidate-processing cost, since at least one candidate is processed
at every step. The substantial decoding costs are model evaluation
and handling its candidate scores.

Importantly, $G_{\max}$ and $A_t$ describe different computations.
The former bounds the grouped Boolean feasibility checks; the latter
counts the exact choices processed by the sampler.
Sharing a feasibility result does not remove the cost of scoring
different components and attachment sites.
We also retain $C_\theta(t)$ explicitly because model evaluation may
depend on the growing prefix. Having $M$ generation steps does not,
by itself, imply $O(M)$ total neural computation.

\paragraph{Preprocessing and total cost.}
Let $T_{\rm pre}$ include construction of the grouped vocabulary index
and auxiliary pruning tables when these are not already available.
These structures can be reused while the vocabulary, count definitions,
and supported component horizon remain unchanged.

Given the template annotations, forming the grouped index takes
$O(d|\mathcal V|)$ time with hashing.
The auxiliary table costs follow from their recurrences.
Let $H$ be the number of distinct bond-type vectors, $U$ the set of
distinct resource vectors, and let $P_f$ and $P_b$ count the frontier
and resource-sum entries expanded across all table levels.
A recurrence-based construction takes
\[
  T_{\rm pre}
  =
  O\!\left(
    d|\mathcal V|+H P_f+d|U|P_b
  \right).
\]
Each frontier entry is extended through the available bond-type vectors,
while each resource-sum entry is extended through $U$.
Optional table limits can reduce this preprocessing work; disabling
a pruning table does not change the exact joint recurrence.

For $B$ molecules sharing a target, $T_{\rm dec}(M)$ upper-bounds each decoding
trajectory:
\[
  T_{\rm total}(B)
  =
  T_{\rm pre}
  +O\!\left(dSG_{\max}+B\,T_{\rm dec}(M)\right).
  \label{eq:total-generation-time}
\]
The main-text expression assumes the reusable preprocessing is already
available. If it must be repeated for a new count definition or a larger
supported horizon, its cost must be included.

When the user does not specify $M$, we compile the candidate component
counts separately. Writing $S_m$ for the states evaluated for component
count $m$, a conservative bound on the total DP work is
\[
  O\!\left(
    dG_{\max}\sum_{m=1}^{M_{\max}}S_m
  \right).
\]
Generation then pays the decoding cost for each selected component count.
This bound does not assume additional sharing between the fixed-$M$
compilations.

\paragraph{Which costs dominate.}
For a new target with few requested samples, compilation can dominate.
Its cost depends on how many distinct states and transitions are explored,
and the state-space bound grows multiplicatively with the resource ranges.
Low resource dimension helps, but does not alone guarantee a cheap target.

For repeated sampling, the DP computation is shared across outputs.
Its amortized contribution is $O(dSG_{\max}/B)$, while the decoding
cost must be paid for every molecule.
With lazy compilation, $S$ may increase as new states are encountered,
but previously solved states are not recomputed.
Once the relevant cache is populated, further generation requires model
prediction and candidate processing rather than another search through
possible completions.

This differs from count-based rejection sampling.
If independent unconstrained attempts satisfy a requested target with
probability $p_{\rm hit}>0$, rejection requires $1/p_{\rm hit}$ attempts
on average for one matching output.
TopoCompile avoids this repetition due to count mismatches, provided
compilation succeeds and feasible choices have positive model probability.
Its advantage is therefore reusable exact completion checking, not a
guarantee of faster generation for every target. Expensive compilation
with little reuse can outweigh the savings from avoiding rejection.

\paragraph{Relation to measured efficiency.}
The reported MOSES scaling experiments illustrate both costs.
First compilation takes $0.034$ seconds in the one-count configuration
and $19.011$ seconds in the three-count configuration.
Their generation throughputs are respectively $7.12\%$ and $28.96\%$
below the corresponding port-only decoder.
These results show that compilation latency and the remaining sampling
overhead must be considered separately.
Figure~\ref{fig:moses-generation-efficiency}(c) assigns sampler/compiler setup, target compilation, reachability,
and cached parent-option checks to DP. The Model category combines model
execution, reconstruction, and final exactness checks; Other overhead contains the
remaining action/state processing and delivery bookkeeping. Cache reuse
reduces DP time from $0.515$ to $0.181$ seconds on MOSES and from $0.670$ to
$0.261$ seconds on ZINC250K, corresponding to $2.85\times$ and $2.57\times$
speedups. Total request time falls from $0.671$ to $0.337$ seconds and from
$0.831$ to $0.421$ seconds, respectively. The measured cold/cache Other
overhead estimates differ by less than $1\%$ within each dataset, so the figure
uses their request-weighted pooled value and moves the compensating difference
to DP, preserving every measured total.

\subsection{Representation implementation and verification}
\label{app:representation-audit}

The implementation first applies the pinned non-isomeric RDKit
canonicalization and serializes atom states, internal bond states, bridge-bond
states, component-local endpoint indices, and port symmetry classes. We
independently recompute depth-first-search bridges
\citep{tarjan1974bridges} and checks agreement with the RDKit-derived component
partition, the bridge--port bijection, and reconstruction of all stored atom
and bond attributes. Thus the validation does not rely on one ring-perception
routine. Disconnected inputs and bond states outside the representation are
excluded.

\subsection{Additivity of cycle rank}
\label{app:cycle-additivity}

\begin{proof}
The components are atom-disjoint, and assembling their bridge tree adds exactly
$M-1$ inter-component edges. Hence $|V|=\sum_i|V_i|$ and
$|E|=\sum_i|E_i|+M-1$, which gives
\begin{equation}
  \beta(G)=|E|-|V|+1
  =\sum_i\bigl(|E_i|-|V_i|+1\bigr)=\sum_i\beta(\tau_i).
  \label{eq:cycle-additivity}
\end{equation}
Thus bridge attachment creates no cross-component cycle, making cycle rank
component-additive.
\end{proof}

\subsection{Structural frontier reachability}

Let $\boldsymbol\sigma(\tau)\in\mathbb N^4$ count the attachment ports of each
bridge-bond type on template $\tau$. Starting from $\mathcal R_0=\{\mathbf0\}$,
define
\begin{equation}
  \mathbf f\in\mathcal R_{r+1}
  \iff \exists\,\ell,\tau:\ f_\ell,\sigma_\ell(\tau)\ge1,\quad
  \mathbf f'=\mathbf f+\boldsymbol\sigma(\tau)-2\mathbf e_\ell
  \in\mathcal R_r.
  \label{eq:reachability}
\end{equation}

\begin{lemma}[Frontier reachability]
\label{lem:reachability}
$\mathbf f\in\mathcal R_r$ if and only if exactly $r$ legal child actions can
consume $\mathbf f$ and terminate at the empty frontier.
\end{lemma}
\begin{proof}
We prove the lemma by induction on $r$. For $r=0$, no child action is available
and the frontier can finish if and only if it is $\mathbf0$. Assume the result
for $r$. If $\mathbf f\in\mathcal R_{r+1}$, the recurrence supplies a type
$\ell$ and template $\tau$ whose attachment reaches
$\mathbf f'\in\mathcal R_r$; the inductive hypothesis supplies the remaining
$r$ actions. Conversely, any legal first child action that admits an $r$-action
completion reaches such a $\mathbf f'$, which the inductive hypothesis places
in $\mathcal R_r$. The recurrence therefore places $\mathbf f$ in
$\mathcal R_{r+1}$.
\end{proof}

\subsection{Joint feasibility DP: soundness and relative completeness}
\label{app:resource-proof}

Fix a finite vocabulary $\mathcal V$, the typed-port transition system in
Equation~\ref{eq:tc-compact-next}, and an admissible map
$r:\mathcal V\to\mathbb N^d$. Equation~\ref{eq:tc-compact-feasibility} defines
$F_{\mathbf c^\star}$ by its terminal condition and Boolean recurrence.

\begin{proof}[Soundness]
For soundness, induct on $k$. The $k=0$ state accepts only
$(\mathbf0,\mathbf0)$. For $k>0$, every retained action subtracts exactly
$r(\tau)$ and takes the state to a reachable $(k-1)$ suffix. The induction
hypothesis closes that suffix with zero frontier and resource residual.
Telescoping the transitions therefore consumes the initial frontier and sums
the selected component contributions to the initial $\mathbf b$. Applying the root check
with $\mathbf b=\mathbf b^\star-r(\tau_1)$ proves exact target equality and
structural completion for every sequence accepted by the joint DP.
\end{proof}

\begin{proof}[Relative completeness]
For relative completeness, consider a feasible $k$-component suffix in the
grammar. Its first component and attachment induce a legal non-negative successor,
whose remaining $k-1$ components form a feasible suffix. Induction gives
$F_{\mathbf c^\star}=1$ for that successor, so the recurrence retains the first
action;
repetition retains the whole sequence. The result is relative to
the finite catalog and transition language. It assumes that every admissible
resource is component-local, non-negative and integer-valued, and unchanged by
bridge attachment. It neither creates support for unseen component interiors nor
guarantees sanitization, synthesis, or learned probability mass. We report
infeasible catalog targets before sampling.
\end{proof}

\subsection{Joint feasibility DP: full algorithm}
\label{app:full-compiler}

For implementation, let $\Gamma_{\rm comp}$ collect the distinct
$(\boldsymbol\sigma(\tau),r(\tau))$ pairs, $\Gamma_\ell$ those with an available
type-$\ell$ port, and $\mathcal U$ the unique resource vectors used below:
\begin{equation}
  \mathcal B_0=\{\mathbf0\},\qquad
  \mathcal B_k=\{\mathbf z+\mathbf u:
  \mathbf z\in\mathcal B_{k-1},\ \mathbf u\in\mathcal U\}.
  \label{eq:resource-only-dp}
\end{equation}
Algorithm~\ref{alg:product-compilation-full} gives the complete execution order
for terminal checks, structural and resource pruning, compatible-group
enumeration, and memoization. If the auxiliary exact-sum cache reaches its cap,
$\operatorname{SumReach}$ conservatively returns \textsc{true}.

\begin{algorithm}[H]
\small
\DontPrintSemicolon
\SetKwProg{Fn}{Function}{:}{}
\SetKw{Continue}{continue}
\caption{Joint feasibility DP with safe pruning}
\label{alg:product-compilation-full}
\KwData{target $(M,\mathbf b^\star)$; catalog index
$\mathcal I=(\Gamma_{\rm tok},\{\Gamma_\ell\}_{\ell=1}^{4},
\mathbf r_{\min},\mathbf r_{\max})$; structural tables
$\{\mathcal R_k\}_{k=0}^{M-1}$}
\KwResult{feasible root groups $\Gamma_{\mathrm{root}}$}
$\mathrm{memo}\leftarrow\emptyset$\;
\Fn{$\operatorname{CanFinish}(k,\mathbf f,\mathbf b)$}{
  $z\leftarrow(k,\mathbf f,\mathbf b)$\;
  \If{$z\in\mathrm{memo}$}{\KwRet $\mathrm{memo}[z]$\;}
  \If{$k=0$}{
    $\mathrm{memo}[z]\leftarrow
      [\mathbf f=\mathbf0\land\mathbf b=\mathbf0]$\;
    \KwRet $\mathrm{memo}[z]$\;
  }
  \If{$k<0\lor\mathbf f\notin\mathcal R_k$}{
    $\mathrm{memo}[z]\leftarrow\textsc{false}$; \KwRet \textsc{false}\;
  }
  \If{$k\mathbf r_{\min}\npreceq\mathbf b\lor
       \mathbf b\npreceq k\mathbf r_{\max}$}{
    $\mathrm{memo}[z]\leftarrow\textsc{false}$; \KwRet \textsc{false}\;
  }
  \If{$\neg\operatorname{SumReach}(k,\mathbf b)$}{
    $\mathrm{memo}[z]\leftarrow\textsc{false}$; \KwRet \textsc{false}\;
  }
  $\mathrm{memo}[z]\leftarrow\textsc{false}$\;
  \ForEach{$\ell$ such that $f_\ell>0$}{
    \ForEach{$\gamma\in\Gamma_\ell$}{
      $\mathbf b'\leftarrow\mathbf b-\mathbf r_\gamma$\;
      \If{$\mathbf b'\not\succeq\mathbf0$}{\Continue\;}
      $\mathbf f'\leftarrow
        \mathbf f+\boldsymbol\sigma_\gamma-2\mathbf e_\ell$\;
      \If{$\mathbf f'\not\succeq\mathbf0\lor
           \mathbf f'\notin\mathcal R_{k-1}$}{\Continue\;}
      \If{$\operatorname{CanFinish}(k-1,\mathbf f',\mathbf b')$}{
        $\mathrm{memo}[z]\leftarrow\textsc{true}$; \KwRet \textsc{true}\;
      }
    }
  }
  \KwRet \textsc{false}\;
}
$\Gamma_{\mathrm{root}}\leftarrow\emptyset$\;
\ForEach{$\gamma\in\Gamma_{\rm tok}$}{
  $\mathbf b'\leftarrow\mathbf b^\star-\mathbf r_\gamma$\;
  \If{$\mathbf b'\succeq\mathbf0\land
       \boldsymbol\sigma_\gamma\in\mathcal R_{M-1}\land
       \operatorname{CanFinish}
       (M-1,\boldsymbol\sigma_\gamma,\mathbf b')$}{
    $\Gamma_{\mathrm{root}}\leftarrow
      \Gamma_{\mathrm{root}}\cup\{\gamma\}$\;
  }
}
\KwRet $\Gamma_{\mathrm{root}}$\;
\end{algorithm}

If the auxiliary cache reaches 1,000,000 states at any level, the implementation
disables only this optional pruning check and continues with the exact joint DP;
this can increase computation but does not remove feasible actions. Reaching
the 2,000,000-state limit of the joint DP instead raises an explicit compilation error
rather than returning an incomplete feasible set. Let $\mathcal F_k$ be the
possible frontier vectors with $k$
components remaining. The target-conditioned state count satisfies
\begin{equation}
  S\le\sum_{k=0}^{M}|\mathcal F_k|
  \prod_{j=1}^{d}(b_j^\star+1),
  \label{eq:product-complexity}
\end{equation}
so the joint DP uses $O(SG_{\max}d)$ worst-case time and $O(S)$ memo memory,
where $G_{\max}$ is the maximum number of compatible groups examined at one
state. Its dependence on target magnitudes and resource dimension is
pseudo-polynomial.
\FloatBarrier

\section{Additional experimental settings}
\label{app:additional-settings}

\subsection{Datasets and preprocessing}
\label{app:data-preprocessing}

We evaluate \textsc{TopoCompile} on MOSES \citep{polykovskiy2020moses},
ZINC250K, GuacaMol \citep{brown2019guacamol}, and NPGen
\citep{lee2026fragfm}. For \textsc{TopoCompile}, preprocessing uses RDKit
2025.09.2 and retains connected molecules within the supported non-isomeric
atom--bond domain. Stereochemical and isotopic identities are not represented.
The component vocabulary is constructed from training molecules only, while
representation coverage is evaluated separately on the test sets. Before
selecting exact-control targets, we remove cross-split molecular overlaps.
Molecular identity in the exact-control evaluation is defined by canonical
non-isomeric connected SMILES, consistently with the training-exclusion set.

\subsection{Model architecture and training}
\label{app:model-training}

\paragraph{Architecture.}
The action Transformer has eight layers, eight attention heads, hidden size
256, and feed-forward size 1,024. Atom and bond embeddings have dimensions
256 and 128, respectively, and component encoding uses a five-layer fragment
encoder. Four prediction heads model the bond-type vector, exact component,
parent port, and child port.

\paragraph{Training.}
We use AdamW with coefficients $(0.9,0.999)$, learning rate $5\times10^{-4}$,
10,000 warm-up updates, no weight decay, exponential moving average decay
0.999, and gradient clipping at 8. Each model is trained for 30,000 updates
with an effective batch size of 256 and eight data workers. All local training
runs reported in this paper use a single Tesla H800 PCIe GPU, and all
experiments report averages over three runs with seeds 123, 456, and 789.

Each randomized rooted traversal produces a ground-truth action sequence.
Under teacher forcing, the model receives the ground-truth partial tree and
predicts the corresponding next action. Parent candidates are unused ports in
the partial tree, component candidates are restricted to the selected
bond-type vector, and child ports must match the parent bond type. The loss
weights for the bond vector, component, parent port, and child port are
$(1,5,1,1)$; only bond-vector and component losses apply at the root. Structural
reachability masks are used during both training and sampling.

\paragraph{Unconditional and constrained generation.}
For unconditional generation, \textsc{TopoCompile} first samples the component
count $M$ from its empirical distribution over the training molecules and
then generates a molecule autoregressively conditioned on $M$. Exact
constrained generation reuses the same checkpoint and applies the joint DP
during sampling, without retraining the generator.

\subsection{Baseline implementations}
\label{app:baseline-implementations}

\paragraph{Graph DiT.}
We use the released conditional diffusion checkpoint of Graph DiT
\citep{liu2024graphdit} and provide the active resource counts as conditioning
inputs for the primary exact-control comparison. When a request specifies
heavy-atom count $N$, we use the requested node count to constrain graph size;
otherwise, we retain the model's learned node-count distribution. We preserve
the released connection repair, valence repair, and largest-component handling
before converting the output to a molecule. The resulting molecule is assessed
by the shared external validity and exact-count evaluator. Each target uses the
fixed attempt budget, including generation, repair, and count-mismatch failures.
Graph DiT receives requested counts without stepwise completion checks.

\paragraph{STGG+ and STGG+ Count Mask.}
We use the released property-conditional spanning-tree generator STGG+
\citep{jolicoeurmartineau2025anyproperty}, conditioning on the active counts
in the primary comparison. Unspecified properties are handled by its native
missing-property mask, and its structural and valence masks are retained.
Guidance strength is selected on separate validation targets. Each attempt
uses one internal candidate, without best-of-$k$ selection or additional
resampling.

STGG+ Count Mask uses the same checkpoint, conditions, guidance, and sampling
budget, adding only a process-local, token-level count mask around the decoder.
After each accepted token, the mask updates the accumulated heavy-atom count,
oxygen count, and cycle rank. Candidate tokens that would exceed any active
target count are masked, and EOS is allowed only when these three counts match
their targets. This mask is combined with the native structural and valence
masks; it performs local count control without checking future completion
feasibility. An empty action set terminates the attempt as a failure. Both
variants use the shared external evaluator. 

\paragraph{FragFM+Guidance.}
We use the released FragFM molecular generator \citep{lee2026fragfm} with a
scalar property predictor for each active resource count. For a requested
target, predictor deviations are normalized using validation-set scales, and
the resulting prediction energies are combined to guide both fragment-bag
generation and molecular graph generation. Bag and graph guidance strengths
are selected separately for each dataset using validation targets. We retain
the official fragment reconstruction procedure without prescribing a
component count, tree layout, or exact completion check. The shared external
validity and exact-count evaluator then assesses each molecule.

\paragraph{GenMol+Filter.}
We use the GenMol V2 unconditional de novo checkpoint
\citep{lee2025genmol} without supplying target counts to the model. For each
target, GenMol independently proposes molecules, and the shared validity and
count checks determine which proposals are exact-valid. There is no retraining,
guidance, or target-specific decoding. Each proposal counts as one attempt,
including invalid outputs and filter rejections; we do not generate replacement
proposals to fill a quota of valid or exact outputs. This implements generation
followed by filtering under a fixed proposal budget.

\paragraph{Tree-Invent.}
For unconditional generation, we use the paper-designated RBFS traversal of
Tree-Invent \citep{xu2023treeinvent}. The GuacaMol row is taken from the
original paper's 50,000-sample evaluation, with V.U. and V.U.N. derived from
its reported validity, uniqueness, and novelty. For MOSES, ZINC250K, and NPGen,
we train the released implementation from scratch for up to six
passes over a deterministic 95\% partition of the training set, reserving the
remaining 5\% for validation. Batch sizes are 320, 144, and 64, respectively,
instead of the released default of 1,000. Each run has an eight-hour training
limit and a four-hour generation-and-evaluation limit. ZINC250K and NPGen
complete, whereas MOSES fails before the first optimizer update and is reported
as OOM. Unconditional-generation tables distinguish local reproductions from
published results.

\paragraph{Other unconditional-generation baselines.}
The unconditional-generation tables also include published results for DiGress
\citep{vignac2023digress}, DeFoG \citep{qin2025defog}, GraphAF
\citep{shi2020graphaf}, and other cited generators. These entries retain their
source attributions and are distinguished from our local Tree-Invent
reproductions.

\paragraph{Guidance selection.}
We select configurations on 32 fixed validation triples disjoint from the
final evaluation targets, using 32 attempts per target. Selection maximizes
exact-valid yield, with validity and exact-output diversity as tie breakers.
For STGG+, we search guidance strengths $\{1.0,1.2,1.5,2.0\}$ at temperature
1. The selected strength is 1.5 on MOSES and GuacaMol and 2.0 on ZINC250K and
NPGen; Count Mask uses the same setting. For FragFM+Guidance, we search
bag/graph strength pairs $(0,2)$, $(0.2,2)$, $(0.4,2)$, and $(1,20)$.
Dataset-specific selection yields $(1,20)$ on all four datasets. No final
evaluation target is used to tune these configurations.

\paragraph{Shared comparison protocol.}
All exact-control methods use the same Common/Rare target lists, main attempt
budgets, external validity and exact-count evaluator, and denominator rules.
Each variant retains its native generation mechanism with only the control
adaptations described above. The comparison therefore measures exact-valid
yield under shared requests and sampling budgets rather than imposing the same
internal architecture or constraint mechanism on every method.

\subsection{Evaluation protocols}
\label{app:evaluation-protocols}

\paragraph{Metric definitions.}
For fixed attempts, Valid is the fraction of chemically valid molecules, and
Exact-valid additionally requires every requested count. For an exact-output
pool, Novelty is the fraction of unique canonical identities absent from
training; internal diversity is one minus mean pairwise Tanimoto similarity of
radius-2, 1024-bit Morgan fingerprints \citep{rogers2010extended}. Conditional
FCD is the Fr\'echet ChemNet distance \citep{preuer2018fcd} between matched
exact outputs and target-matched references. For unconditional generation,
Validity is the valid fraction under each benchmark denominator, and V.U.N. is
validity $\times$ uniqueness $\times$ novelty. FCD compares ChemNet embeddings
of valid outputs with the benchmark reference set; GuacaMol reports the
corresponding score, where higher is better.

\paragraph{Unconditional-generation budgets and replication.}
MOSES and ZINC250K each use 25,000 generation attempts, and NPGen uses 30,000.
GuacaMol uses an 11,000-molecule pool to construct its 10,000-molecule report.
\textsc{TopoCompile} results on MOSES, ZINC250K, GuacaMol, and NPGen average
three independently trained models. Each table or caption states its repetition
unit.

\paragraph{Benchmark-specific evaluation.}
Validity on MOSES, ZINC250K, and NPGen uses all attempts. GuacaMol validity
uses the first 10,000 attempts; its FCD uses the first 10,000 pooled-valid
molecules, while KL, novelty, and V.U.N. use the first 10,000 pooled unique
molecules. Other FCD and ring metrics use valid canonical molecules. MOSES
FCD/TestSF uses the 176,225-molecule scaffold test split, and GuacaMol uses
its training reference and published score definition. ZINC250K uses the
final 5,000 test molecules with a local implementation of the benchmark
evaluator. NPGen uses its released test embedding.

\paragraph{Exact-control requests and yields.}
The primary comparison evaluates $(N,\beta,N_{\mathrm O},R_{\rm sys})$
without supplying $M$; lower-dimensional schemas are reported in the
appendix. \textsc{TopoCompile} checks supported candidate
component counts, retains those admitting a compiler-feasible completion, and
samples from the restricted empirical prior $p(M\mid N)$ when $N$ is given, or
$p(M)$ otherwise. Exact-control experiments use frozen checkpoints and fixed
sampling streams, with $K=32$ attempts per target unless stated otherwise.

Exact-valid yield (Exact) is the fraction of attempts producing valid molecules
that satisfy all requested counts. Novel unique exact yield (NUE) counts the
exact-valid canonical molecules that are absent from training and deduplicated
within each target and sampling run, using the same attempt denominator. We
average target yields by schema and tier, then macro-average schemas.

\paragraph{Matched-output quality.}
For the main exact-control quality comparison, we use preselected Common
targets and the first 100 attempts per target. Compared methods contribute
equal-sized exact-output pools for each target, with survivors selected in
attempt order. A target is excluded when any compared method provides fewer
than two exact outputs. Novelty is the fraction of unique matched identities
absent from training. Internal diversity is one minus the mean pairwise
Tanimoto similarity of radius-2, 1024-bit Morgan fingerprints. Graph DiT and
GenMol+Filter are unreported where their exact outputs do not support a
compatible matched pool.

\paragraph{Ring-based metrics.}
For a parsed molecular graph, cycle rank is $|E|-|V|+c$, where $c$ is the
number of connected components and equals one in the supported compiler domain.
Ring systems follow the maximal non-bridge component definition used by the
representation, and ring size is the multiset of RDKit atom-ring sizes. The
topology signature comprises nonexclusive fused, spiro, and bridgehead
indicators: fused denotes a bond belonging to multiple perceived atom rings,
while spiro and bridgehead indicators use RDKit's canonical counters. Core
identity is the port-free canonical component SMILES, and a core pair is an
unordered pair of core identities within one molecule. For each of these six
statistics, we compute $S_k=1-\operatorname{JSD}_{\mathrm{nats}}(P_k,Q_k)/\log 2$;
RingSF is their mean.

\paragraph{NPGen proxy metrics.}
NPGen's NP-score and classifier outputs are learned proxies. Its released
classifier KLD discards generated-only labels and renormalizes on reference
support. We retain this definition for benchmark compatibility and also report
KLD over the union support after probability normalization, together with
generated-only mass. The reported benchmark proxy values reproduce within
$10^{-6}$ using the released weights and source files.

\subsection{Representation and compiler verification}
\label{app:implementation-verification}

We verify the implementation on all 4,243,193 connected records retained after
preprocessing the training, validation, and test splits of the four datasets.
Independent Tarjan bridge detection \citep{tarjan1974bridges} is checked against
RDKit-derived ring grouping. The checks cover component-node types,
component-tree connectivity and edge count, cycle localization, the bridge--port
bijection, template identity, feature-preserving graph reconstruction, and
catalog consistency. No violations are found. Testing 4,096 atom-order
permutations confirms preservation of topology, ports, and reconstruction.

The joint DP is independently checked against exhaustive enumeration on
five small vocabularies for each resource dimension $d\in\{1,2,3\}$. These
comparisons remove no feasible completions and retain no infeasible completions.
Independent recounting finds no count mismatch in 10,000 MOSES and 10,000
ZINC250K molecules.

\FloatBarrier
\section{Exact-Control Settings and Schema-Specific Results}
\label{app:shared-control}
\renewcommand{\thetable}{S\arabic{table}}
\renewcommand{\theHtable}{supplement.\arabic{table}}
\setcounter{table}{0}
\renewcommand{\thefigure}{S\arabic{figure}}
\renewcommand{\theHfigure}{supplement.\arabic{figure}}
\setcounter{figure}{0}

This section specifies how we construct exact-control requests, configure the
compared methods, and compute the reported success and quality metrics. It then
presents the schema-specific results and examines how exact-control performance
changes with target frequency.

\paragraph{Requests and target groups.}
Each request specifies a subset of heavy-atom count $N$, cycle rank $\beta$,
oxygen count $N_{\rm O}$, and ring-system count $R_{\rm sys}$, but does not
specify the component count $M$.
TopoCompile therefore enumerates candidate values of $M$, retains those with a
compiler-feasible completion, and samples from $p(M\mid N)$ when the request specifies $N$
or from $p(M)$ otherwise. We remove cross-split molecular overlaps before
selecting targets, so the resulting Common and Rare groups preserve their
observed target identities without sharing molecules across splits. For the
frequency analysis, we assign equal-frequency tuples their average rank among
all distinct targets. Figure~\ref{fig:frequency-response}(A) divides these ranks
into four percentile ranges and selects eight four-count tuples, spaced
uniformly by rank, from each range.

\paragraph{Baselines and configuration selection.}
We compare TopoCompile with Graph DiT, STGG+, FragFM+Guidance,
GenMol+Filter, and STGG+ Count Mask under a shared external count
check. Graph DiT forces the requested node count only for requests that specify $N$ and
retains its released connection/valence repair and largest-component handling.
STGG+ retains its missing-property and structural/valence masks, whereas
FragFM+Guidance applies normalized predictor energies for the active properties
during bag and graph generation. GenMol+Filter instead generates unconditional
candidates and filters them independently. Finally, STGG+ Count Mask reuses the STGG+
checkpoint and blocks count overshoots and premature EOS; an exhausted action
set terminates as a failed attempt. Because guidance strength directly affects
these comparisons, we select STGG+ guidance from
$\{1.0,1.2,1.5,2.0\}$ at temperature 1 and FragFM bag/graph strengths from
$(0,2)$, $(0.2,2)$, $(0.4,2)$, and $(1,20)$ on separate validation targets.
Selection maximizes exact-valid yield at a fixed candidate budget, using
validity and exact-output diversity as tie breakers. We select one configuration
per dataset, and STGG+ uses one internal candidate per attempt. This procedure
selects STGG+ guidance 1.5 for MOSES and GuacaMol and 2.0 for ZINC250K and
NPGen; STGG+ Count Mask uses the corresponding STGG+ setting. FragFM uses bag/graph
strengths 1/20 on every dataset, selected using 32 attempts on each of 32 fixed
validation triples disjoint from the final requests.

\paragraph{Metrics, denominators, and aggregation.}
We first evaluate every generated molecule with the same external validity and
count checks. Valid yield is the fraction of attempts that produce a valid
molecule, while Exact-valid yield is the fraction that are both valid and exact;
both metrics use the full attempt count as their denominator. Consequently, an
invalid output remains in the denominator, whereas a request that cannot finish
compilation produces no sampled candidates; we record it separately as
request noncompletion. We also distinguish state-limit exhaustion, timeout,
and certified infeasibility. Molecular identity throughout is canonical
non-isomeric connected SMILES, consistent with the training exclusion set.
After computing the joint events per output, we average targets within each
schema and macro-average schemas where reported. For matched quality, we use
the same preselected Common targets and the first 100 attempts per target.
Every method contributes the same number of exact survivors in attempt order,
and we exclude a target when any compared method supplies fewer than two.
Novelty is then the fraction of unique matched identities absent from training,
and internal diversity is one minus mean pairwise Tanimoto similarity using
radius-2, 1024-bit Morgan fingerprints. Graph DiT and GenMol+Filter remain
unreported when their exact-output counts do not support a compatible pool.

All selected TopoCompile requests entered sampling and completed the full
evaluation pipeline.
Tables~\ref{tab:exact-control-n}--\ref{tab:exact-control-beta-oxygen} provide
lower-dimensional complements to the main four-count comparison for three
schemas with near-complete comparison grids. The sparser $N_{\rm O}$, $(N,\beta)$, and
$(N,N_{\rm O})$ grids are not repeated as standalone tables.
All four tables use the fixed sampling protocol. Panel A reports Valid and Exact-valid in
percent for the first $K=32$ Common and Rare attempts. In Tables~\ref{tab:exact-control-n}--
\ref{tab:exact-control-beta-oxygen}, Panel B reports Novelty and internal diversity
from matched exact outputs among the first 100 Common attempts;
Table~\ref{tab:four-count-extension} uses the same budget. For the four-count table, the
MOSES/ZINC250K/GuacaMol/NPGen comparable-target and matched-sample counts are
32/690, 32/374, 32/172, and 15/45, respectively. In
MOSES/ZINC250K/GuacaMol/NPGen order, dashes denote missing entries and the
comparable target/sample counts are 6/122, 9/104, 20/127, and 23/127 for $N$; 5/434,
5/396, 5/257, and 7/345 for $\beta$; and 18/1,179, 21/1,177, 32/840, and
32/429 for $(\beta,N_{\rm O})$.
\renewcommand{\thetable}{S1a}
\begin{table}[H]
\centering
\caption{Numbers of shared Common and Rare targets by dataset and constraint
schema.}
\label{tab:shared-targets}
\scriptsize
\setlength{\tabcolsep}{2.5pt}
\begin{tabular*}{\textwidth}{@{\extracolsep{\fill}}lccccccc@{}}
\toprule
Dataset & $N$ & $\beta$ & $N_{\rm O}$ & $(N,\beta)$ & $(N,N_{\rm O})$
& $(\beta,N_{\rm O})$ & $(N,\beta,N_{\rm O})$ \\
\midrule
MOSES & 6/6 & 5/4 & 6/6 & 26/26 & 32/32 & 18/18 & 32/32 \\
ZINC250K & 9/9 & 5/4 & 6/5 & 32/32 & 32/32 & 21/22 & 32/32 \\
GuacaMol & 21/21 & 5/5 & 8/8 & 32/32 & 32/32 & 32/32 & 32/32 \\
NPGen & 25/25 & 7/7 & 13/13 & 32/32 & 32/32 & 32/32 & 32/32 \\
\bottomrule
\end{tabular*}
\end{table}

\renewcommand{\thetable}{S2a}
\begin{table}[H]
\centering
\caption{Valid and E.V (Exact-valid) yields (\%) for heavy-atom count $N$.}
\label{tab:exact-control-n}
\scriptsize
\setlength{\tabcolsep}{0.7pt}
\renewcommand{\arraystretch}{0.94}
\begin{tabularx}{\textwidth}{@{}l!{\vrule width 0.6pt}YYYY!{\vrule width 0.6pt}YYYY!{\vrule width 0.6pt}YYYY!{\vrule width 0.6pt}YYYY@{}}
\toprule
\textbf{Method}
& \multicolumn{4}{c!{\vrule width 0.6pt}}{\textbf{MOSES}}
& \multicolumn{4}{c!{\vrule width 0.6pt}}{\textbf{ZINC250K}}
& \multicolumn{4}{c!{\vrule width 0.6pt}}{\textbf{GuacaMol}}
& \multicolumn{4}{c}{\textbf{NPGen}} \\
\rowcolor{tablegray}
& \multicolumn{2}{c}{\textit{Common}} & \multicolumn{2}{c!{\vrule width 0.6pt}}{\textit{Rare}}
& \multicolumn{2}{c}{\textit{Common}} & \multicolumn{2}{c!{\vrule width 0.6pt}}{\textit{Rare}}
& \multicolumn{2}{c}{\textit{Common}} & \multicolumn{2}{c!{\vrule width 0.6pt}}{\textit{Rare}}
& \multicolumn{2}{c}{\textit{Common}} & \multicolumn{2}{c}{\textit{Rare}} \\
\rowcolor{tablegray}
& Valid & E.V & Valid & E.V & Valid & E.V & Valid & E.V
& Valid & E.V & Valid & E.V & Valid & E.V & Valid & E.V \\
\midrule
Graph DiT & 67.19 & 67.19 & 31.77 & 31.77 & 61.81 & 61.81 & 46.18 & 46.18 & 12.80 & 12.80 & 8.48 & 8.48 & 44.63 & 44.63 & 35.88 & 35.88 \\
STGG+ & 100.00 & 93.06 & 100.00 & 21.01 & 100.00 & 92.71 & 100.00 & 30.90 & 100.00 & 77.98 & 100.00 & 7.14 & 100.00 & 61.63 & 99.88 & 10.88 \\
FragFM+Guidance & 97.92 & 26.82 & 40.89 & 4.43 & 100.00 & 15.97 & 97.22 & 1.74 & 93.30 & 5.51 & 91.67 & 0.15 & 98.25 & 5.13 & 92.50 & 0.13 \\
GenMol+Filter & 100.00 & 16.67 & 100.00 & 0.00 & 100.00 & 11.11 & 100.00 & 0.00 & 100.00 & 5.80 & 100.00 & 0.00 & 100.00 & 3.88 & 100.00 & 0.00 \\
STGG+ Count Mask & 96.35 & 96.35 & 24.48 & 24.48 & 93.40 & 93.40 & 46.53 & 46.53 & 83.18 & 83.18 & 25.15 & 25.15 & 70.00 & 70.00 & 27.13 & 27.13 \\
\rowcolor{cyan!15}
\textbf{TopoCompile (ours)} & \textbf{100.00} & \textbf{100.00} & \textbf{100.00} & \textbf{100.00} & \textbf{100.00} & \textbf{100.00} & \textbf{100.00} & \textbf{100.00} & \textbf{100.00} & \textbf{100.00} & \textbf{100.00} & \textbf{100.00} & \textbf{100.00} & \textbf{100.00} & \textbf{100.00} & \textbf{100.00} \\
\bottomrule
\end{tabularx}

\vspace{1.5pt}
\begin{tabularx}{\textwidth}{@{}l!{\vrule width 0.6pt}YY!{\vrule width 0.6pt}YY!{\vrule width 0.6pt}YY!{\vrule width 0.6pt}YY@{}}
\toprule
\multicolumn{9}{@{}l}{\textit{B. Matched exact-output novelty and diversity}} \\
\rowcolor{tablegray}
\textbf{Method}
& \multicolumn{2}{c!{\vrule width 0.6pt}}{\textbf{MOSES}}
& \multicolumn{2}{c!{\vrule width 0.6pt}}{\textbf{ZINC250K}}
& \multicolumn{2}{c!{\vrule width 0.6pt}}{\textbf{GuacaMol}}
& \multicolumn{2}{c}{\textbf{NPGen}} \\
\rowcolor{tablegray}
& Novelty $\uparrow$ & Div. $\uparrow$ & Novelty $\uparrow$ & Div. $\uparrow$
& Novelty $\uparrow$ & Div. $\uparrow$ & Novelty $\uparrow$ & Div. $\uparrow$ \\
\midrule
STGG+ & 0.910 & 0.854 & 0.990 & 0.870 & 0.984 & 0.876 & 0.992 & 0.882 \\
FragFM+Guidance & 0.893 & 0.863 & 0.981 & 0.870 & 0.937 & 0.876 & 0.866 & 0.886 \\
STGG+ Count Mask & 0.910 & 0.854 & 1.000 & 0.872 & 0.984 & 0.877 & 0.992 & 0.883 \\
\rowcolor{cyan!15}
\textbf{TopoCompile (ours)} & \textbf{0.951} & \textbf{0.864} & \textbf{0.990} & \textbf{0.868} & \textbf{0.992} & \textbf{0.878} & \textbf{0.984} & \textbf{0.889} \\
\bottomrule
\end{tabularx}
\end{table}

\renewcommand{\thetable}{S2b}
\begin{table}[H]
\centering
\caption{Valid and E.V (Exact-valid) yields (\%) for cycle rank $\beta$.}
\label{tab:exact-control-beta}
\scriptsize
\setlength{\tabcolsep}{0.7pt}
\renewcommand{\arraystretch}{0.94}
\begin{tabularx}{\textwidth}{@{}l!{\vrule width 0.6pt}YYYY!{\vrule width 0.6pt}YYYY!{\vrule width 0.6pt}YYYY!{\vrule width 0.6pt}YYYY@{}}
\toprule
\textbf{Method}
& \multicolumn{4}{c!{\vrule width 0.6pt}}{\textbf{MOSES}}
& \multicolumn{4}{c!{\vrule width 0.6pt}}{\textbf{ZINC250K}}
& \multicolumn{4}{c!{\vrule width 0.6pt}}{\textbf{GuacaMol}}
& \multicolumn{4}{c}{\textbf{NPGen}} \\
\rowcolor{tablegray}
& \multicolumn{2}{c}{\textit{Common}} & \multicolumn{2}{c!{\vrule width 0.6pt}}{\textit{Rare}}
& \multicolumn{2}{c}{\textit{Common}} & \multicolumn{2}{c!{\vrule width 0.6pt}}{\textit{Rare}}
& \multicolumn{2}{c}{\textit{Common}} & \multicolumn{2}{c!{\vrule width 0.6pt}}{\textit{Rare}}
& \multicolumn{2}{c}{\textit{Common}} & \multicolumn{2}{c}{\textit{Rare}} \\
\rowcolor{tablegray}
& Valid & E.V & Valid & E.V & Valid & E.V & Valid & E.V
& Valid & E.V & Valid & E.V & Valid & E.V & Valid & E.V \\
\midrule
Graph DiT & 54.17 & 41.25 & 40.62 & 5.21 & 45.00 & 29.58 & 35.42 & 10.68 & 23.12 & 5.63 & 69.38 & 1.25 & 41.52 & 9.38 & 15.62 & 0.00 \\
STGG+ & 100.00 & 94.38 & 100.00 & 32.55 & 100.00 & 95.00 & 100.00 & 65.63 & 100.00 & 98.75 & 100.00 & 4.38 & 100.00 & 95.54 & -- & -- \\
FragFM+Guidance & 98.75 & 88.44 & 65.23 & 15.23 & 100.00 & 86.25 & 94.53 & 26.56 & 90.00 & 49.38 & 71.25 & 2.50 & 96.88 & 53.57 & 41.07 & 0.45 \\
GenMol+Filter & 100.00 & 18.75 & 100.00 & 0.00 & 100.00 & 16.88 & 100.00 & 0.00 & 100.00 & 20.00 & 100.00 & 0.00 & 100.00 & 14.29 & 100.00 & 0.00 \\
STGG+ Count Mask & 96.25 & 95.63 & 33.59 & 32.81 & 95.00 & 95.00 & 71.88 & 67.97 & 98.75 & 98.75 & 23.75 & 16.25 & 95.54 & 95.54 & 41.07 & 27.23 \\
\rowcolor{cyan!15}
\textbf{TopoCompile (ours)} & \textbf{100.00} & \textbf{100.00} & \textbf{100.00} & \textbf{100.00} & \textbf{100.00} & \textbf{100.00} & \textbf{100.00} & \textbf{100.00} & \textbf{100.00} & \textbf{100.00} & \textbf{100.00} & \textbf{100.00} & \textbf{100.00} & \textbf{100.00} & \textbf{100.00} & \textbf{100.00} \\
\bottomrule
\end{tabularx}

\vspace{1.5pt}
\begin{tabularx}{\textwidth}{@{}l!{\vrule width 0.6pt}YY!{\vrule width 0.6pt}YY!{\vrule width 0.6pt}YY!{\vrule width 0.6pt}YY@{}}
\toprule
\multicolumn{9}{@{}l}{\textit{B. Matched exact-output novelty and diversity}} \\
\rowcolor{tablegray}
\textbf{Method}
& \multicolumn{2}{c!{\vrule width 0.6pt}}{\textbf{MOSES}}
& \multicolumn{2}{c!{\vrule width 0.6pt}}{\textbf{ZINC250K}}
& \multicolumn{2}{c!{\vrule width 0.6pt}}{\textbf{GuacaMol}}
& \multicolumn{2}{c}{\textbf{NPGen}} \\
\rowcolor{tablegray}
& Novelty $\uparrow$ & Div. $\uparrow$ & Novelty $\uparrow$ & Div. $\uparrow$
& Novelty $\uparrow$ & Div. $\uparrow$ & Novelty $\uparrow$ & Div. $\uparrow$ \\
\midrule
STGG+ & 0.938 & 0.871 & 1.000 & 0.879 & 0.996 & 0.878 & 0.855 & 0.882 \\
FragFM+Guidance & 0.896 & 0.864 & 0.995 & 0.872 & 0.957 & 0.881 & 0.913 & 0.884 \\
STGG+ Count Mask & 0.931 & 0.871 & 0.997 & 0.880 & 0.996 & 0.878 & 0.855 & 0.882 \\
\rowcolor{cyan!15}
\textbf{TopoCompile (ours)} & \textbf{0.975} & \textbf{0.868} & \textbf{0.995} & \textbf{0.878} & \textbf{0.996} & \textbf{0.884} & \textbf{0.977} & \textbf{0.896} \\
\bottomrule
\end{tabularx}
\end{table}

\renewcommand{\thetable}{S2c}
\begin{table}[H]
\centering
\caption{Valid and E.V (Exact-valid) yields (\%) for cycle rank and oxygen count $(\beta,N_{\rm O})$.}
\label{tab:exact-control-beta-oxygen}
\scriptsize
\setlength{\tabcolsep}{0.7pt}
\renewcommand{\arraystretch}{0.94}
\begin{tabularx}{\textwidth}{@{}l!{\vrule width 0.6pt}YYYY!{\vrule width 0.6pt}YYYY!{\vrule width 0.6pt}YYYY!{\vrule width 0.6pt}YYYY@{}}
\toprule
\textbf{Method}
& \multicolumn{4}{c!{\vrule width 0.6pt}}{\textbf{MOSES}}
& \multicolumn{4}{c!{\vrule width 0.6pt}}{\textbf{ZINC250K}}
& \multicolumn{4}{c!{\vrule width 0.6pt}}{\textbf{GuacaMol}}
& \multicolumn{4}{c}{\textbf{NPGen}} \\
\rowcolor{tablegray}
& \multicolumn{2}{c}{\textit{Common}} & \multicolumn{2}{c!{\vrule width 0.6pt}}{\textit{Rare}}
& \multicolumn{2}{c}{\textit{Common}} & \multicolumn{2}{c!{\vrule width 0.6pt}}{\textit{Rare}}
& \multicolumn{2}{c}{\textit{Common}} & \multicolumn{2}{c!{\vrule width 0.6pt}}{\textit{Rare}}
& \multicolumn{2}{c}{\textit{Common}} & \multicolumn{2}{c}{\textit{Rare}} \\
\rowcolor{tablegray}
& Valid & E.V & Valid & E.V & Valid & E.V & Valid & E.V
& Valid & E.V & Valid & E.V & Valid & E.V & Valid & E.V \\
\midrule
Graph DiT & 57.52 & 36.69 & 53.24 & 1.27 & 53.32 & 24.31 & 61.93 & 3.22 & 21.00 & 1.37 & 35.35 & 0.00 & 35.16 & 2.25 & 35.84 & 0.00 \\
STGG+ & 100.00 & 97.51 & 100.00 & 9.14 & 100.00 & 96.43 & 100.00 & 14.77 & 100.00 & 96.09 & 100.00 & 0.68 & 100.00 & 89.55 & 92.09 & 1.17 \\
FragFM+Guidance & 99.74 & 66.15 & 75.78 & 5.90 & 100.00 & 55.80 & 91.48 & 5.68 & 91.11 & 26.27 & 70.51 & 0.10 & 98.14 & 14.45 & 66.11 & 0.00 \\
GenMol+Filter & 100.00 & 5.03 & 100.00 & 0.00 & 100.00 & 5.36 & 100.00 & 0.00 & 100.00 & 3.13 & 100.00 & 0.00 & 100.00 & 1.37 & 100.00 & 0.00 \\
STGG+ Count Mask & 97.40 & 97.22 & 44.79 & 39.06 & 97.47 & 97.32 & 49.57 & 47.16 & 97.07 & 96.88 & 17.68 & 14.45 & 93.75 & 93.46 & 62.50 & 55.57 \\
\rowcolor{cyan!15}
\textbf{TopoCompile (ours)} & \textbf{100.00} & \textbf{100.00} & \textbf{100.00} & \textbf{100.00} & \textbf{100.00} & \textbf{100.00} & \textbf{100.00} & \textbf{100.00} & \textbf{100.00} & \textbf{100.00} & \textbf{100.00} & \textbf{100.00} & \textbf{100.00} & \textbf{100.00} & -- & -- \\
\bottomrule
\end{tabularx}

\vspace{1.5pt}
\begin{tabularx}{\textwidth}{@{}l!{\vrule width 0.6pt}YY!{\vrule width 0.6pt}YY!{\vrule width 0.6pt}YY!{\vrule width 0.6pt}YY@{}}
\toprule
\multicolumn{9}{@{}l}{\textit{B. Matched exact-output novelty and diversity}} \\
\rowcolor{tablegray}
\textbf{Method}
& \multicolumn{2}{c!{\vrule width 0.6pt}}{\textbf{MOSES}}
& \multicolumn{2}{c!{\vrule width 0.6pt}}{\textbf{ZINC250K}}
& \multicolumn{2}{c!{\vrule width 0.6pt}}{\textbf{GuacaMol}}
& \multicolumn{2}{c}{\textbf{NPGen}} \\
\rowcolor{tablegray}
& Novelty $\uparrow$ & Div. $\uparrow$ & Novelty $\uparrow$ & Div. $\uparrow$
& Novelty $\uparrow$ & Div. $\uparrow$ & Novelty $\uparrow$ & Div. $\uparrow$ \\
\midrule
STGG+ & 0.933 & 0.865 & 0.997 & 0.878 & 0.992 & 0.882 & 0.944 & 0.887 \\
FragFM+Guidance & 0.922 & 0.860 & 0.996 & 0.873 & 0.989 & 0.878 & 0.911 & 0.880 \\
STGG+ Count Mask & 0.932 & 0.865 & 0.998 & 0.881 & 0.992 & 0.882 & 0.944 & 0.887 \\
\rowcolor{cyan!15}
\textbf{TopoCompile (ours)} & \textbf{0.990} & \textbf{0.871} & \textbf{1.000} & \textbf{0.883} & \textbf{0.998} & \textbf{0.883} & \textbf{0.995} & \textbf{0.894} \\
\bottomrule
\end{tabularx}
\end{table}

\renewcommand{\thetable}{S2d}
\begin{table}[H]
\centering
\caption{Valid and E.V (Exact-valid) yields (\%) for the four-count request
$(N,\beta,N_{\rm O},R_{\rm sys})$. Dashes in the STGG+ Count Mask row denote Valid
values not retained by the compact export; its exact-valid values are complete.}
\label{tab:four-count-extension}
\scriptsize
\setlength{\tabcolsep}{0.7pt}
\renewcommand{\arraystretch}{0.94}
\begin{tabularx}{\textwidth}{@{}l!{\vrule width 0.6pt}YYYY!{\vrule width 0.6pt}YYYY!{\vrule width 0.6pt}YYYY!{\vrule width 0.6pt}YYYY@{}}
\toprule
\textbf{Method}
& \multicolumn{4}{c!{\vrule width 0.6pt}}{\textbf{MOSES}}
& \multicolumn{4}{c!{\vrule width 0.6pt}}{\textbf{ZINC250K}}
& \multicolumn{4}{c!{\vrule width 0.6pt}}{\textbf{GuacaMol}}
& \multicolumn{4}{c}{\textbf{NPGen}} \\
\rowcolor{tablegray}
& \multicolumn{2}{c}{\textit{Common}} & \multicolumn{2}{c!{\vrule width 0.6pt}}{\textit{Rare}}
& \multicolumn{2}{c}{\textit{Common}} & \multicolumn{2}{c!{\vrule width 0.6pt}}{\textit{Rare}}
& \multicolumn{2}{c}{\textit{Common}} & \multicolumn{2}{c!{\vrule width 0.6pt}}{\textit{Rare}}
& \multicolumn{2}{c}{\textit{Common}} & \multicolumn{2}{c}{\textit{Rare}} \\
\rowcolor{tablegray}
& Valid & E.V & Valid & E.V & Valid & E.V & Valid & E.V
& Valid & E.V & Valid & E.V & Valid & E.V & Valid & E.V \\
\midrule
Graph DiT & 69.34 & 56.64 & 41.60 & 5.57 & 66.02 & 43.36 & 51.07 & 8.40 & 3.61 & 0.49 & 49.51 & 4.30 & \multicolumn{4}{c}{\textsc{oom}} \\
STGG+ & 100.00 & 80.86 & 100.00 & 29.00 & 100.00 & 79.59 & 100.00 & 37.50 & 100.00 & 34.86 & 100.00 & 13.87 & 100.00 & 16.89 & 100.00 & 3.71 \\
FragFM+Guidance & 97.85 & 22.07 & 69.34 & 0.20 & 99.12 & 11.04 & 97.46 & 0.00 & 93.16 & 4.98 & 90.53 & 0.00 & 98.34 & 1.56 & 97.46 & 0.00 \\
GenMol+Filter & 100.00 & 1.86 & 100.00 & 0.00 & 100.00 & 1.46 & 100.00 & 0.00 & 100.00 & 1.66 & 99.90 & 0.00 & 100.00 & 0.00 & 100.00 & 0.00 \\
STGG+ Count Mask & -- & 83.59 & -- & 35.55 & -- & 82.71 & -- & 44.92 & -- & 46.88 & -- & 15.53 & -- & 27.44 & -- & 14.06 \\
\rowcolor{cyan!15}
\textbf{TopoCompile (ours)} & \textbf{100.00} & \textbf{100.00} & \textbf{100.00} & \textbf{100.00} & \textbf{100.00} & \textbf{100.00} & \textbf{100.00} & \textbf{100.00} & \textbf{100.00} & \textbf{100.00} & \textbf{100.00} & \textbf{100.00} & \textbf{100.00} & \textbf{100.00} & \textbf{100.00} & \textbf{100.00} \\
\bottomrule
\end{tabularx}

\vspace{1.5pt}
\begin{tabularx}{\textwidth}{@{}l!{\vrule width 0.6pt}YY!{\vrule width 0.6pt}YY!{\vrule width 0.6pt}YY!{\vrule width 0.6pt}YY@{}}
\toprule
\multicolumn{9}{@{}l}{\textit{B. Matched exact-output novelty and diversity from the first 100 attempts}} \\
\rowcolor{tablegray}
\textbf{Method}
& \multicolumn{2}{c!{\vrule width 0.6pt}}{\textbf{MOSES}}
& \multicolumn{2}{c!{\vrule width 0.6pt}}{\textbf{ZINC250K}}
& \multicolumn{2}{c!{\vrule width 0.6pt}}{\textbf{GuacaMol}}
& \multicolumn{2}{c}{\textbf{NPGen}} \\
\rowcolor{tablegray}
& Novelty $\uparrow$ & Div. $\uparrow$ & Novelty $\uparrow$ & Div. $\uparrow$
& Novelty $\uparrow$ & Div. $\uparrow$ & Novelty $\uparrow$ & Div. $\uparrow$ \\
\midrule
STGG+ & 0.946 & 0.845 & \textbf{1.000} & 0.860 & 0.994 & 0.867 & \textbf{0.978} & 0.883 \\
FragFM+Guidance & 0.899 & 0.846 & 0.987 & 0.852 & 0.988 & 0.858 & 0.689 & 0.860 \\
STGG+ Count Mask & 0.952 & 0.847 & \textbf{1.000} & 0.861 & 0.994 & 0.869 & \textbf{0.978} & 0.883 \\
\rowcolor{cyan!15}
\textbf{TopoCompile (ours)} & \textbf{0.990} & \textbf{0.862} & \textbf{1.000} & \textbf{0.872} & \textbf{1.000} & \textbf{0.882} & 0.956 & \textbf{0.900} \\
\bottomrule
\end{tabularx}
\end{table}

\FloatBarrier
\begin{samepage}
\noindent
\begin{minipage}[t]{0.485\textwidth}
\centering
\includegraphics[width=\linewidth]{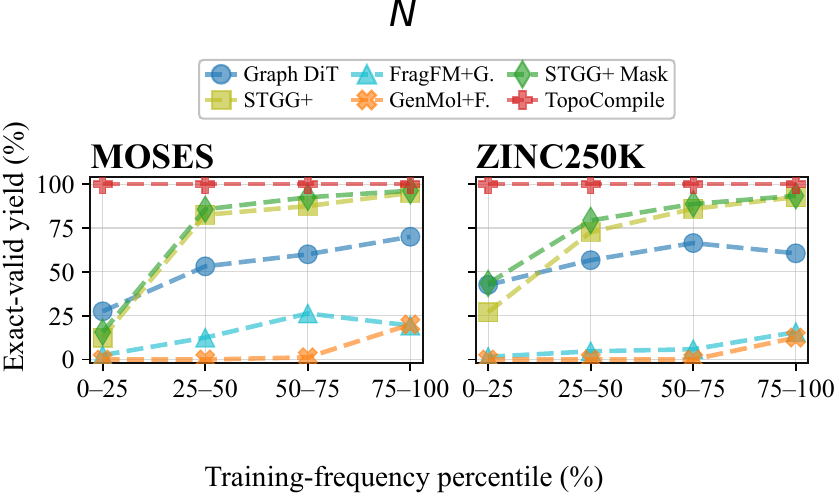}
\refstepcounter{figure}
\label{fig:full-frequency-response-n}
\par\small\raggedright
\textbf{Figure~\thefigure:} Exact-valid yield for the $N$ schema on MOSES and ZINC250K.
\end{minipage}\hfill
\begin{minipage}[t]{0.485\textwidth}
\centering
\includegraphics[width=\linewidth]{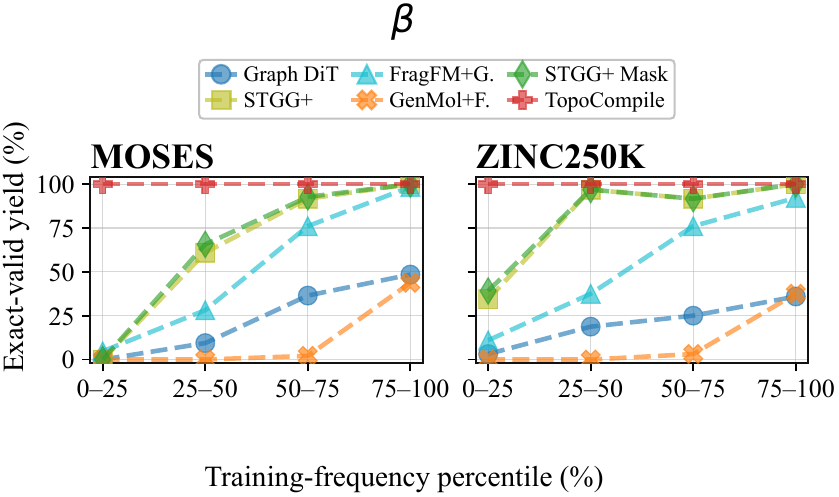}
\refstepcounter{figure}
\label{fig:full-frequency-response-beta}
\par\small\raggedright
\textbf{Figure~\thefigure:} Exact-valid yield for the $\beta$ schema on MOSES and ZINC250K.
\end{minipage}
\par\vspace{0.18em}

\noindent
\begin{minipage}[t]{0.485\textwidth}
\centering
\includegraphics[width=\linewidth]{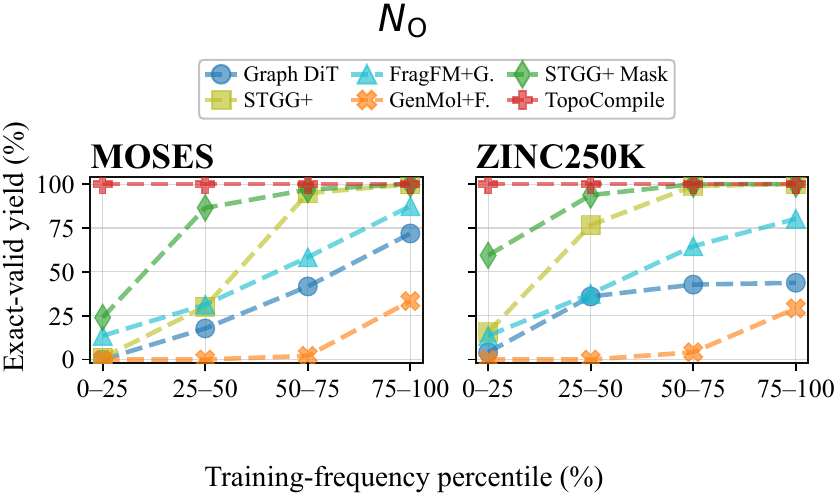}
\refstepcounter{figure}
\label{fig:full-frequency-response-oxygen}
\par\small\raggedright
\textbf{Figure~\thefigure:} Exact-valid yield for the $N_{\mathrm{O}}$ schema on MOSES and ZINC250K.
\end{minipage}\hfill
\begin{minipage}[t]{0.485\textwidth}
\centering
\includegraphics[width=\linewidth]{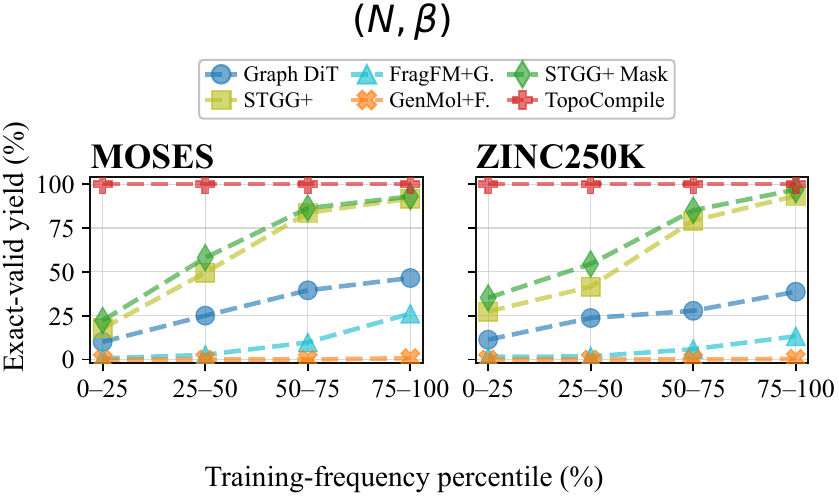}
\refstepcounter{figure}
\label{fig:full-frequency-response-n_beta}
\par\small\raggedright
\textbf{Figure~\thefigure:} Exact-valid yield for the $(N,\beta)$ schema on MOSES and ZINC250K.
\end{minipage}
\par\vspace{0.18em}

\noindent
\begin{minipage}[t]{0.485\textwidth}
\centering
\includegraphics[width=\linewidth]{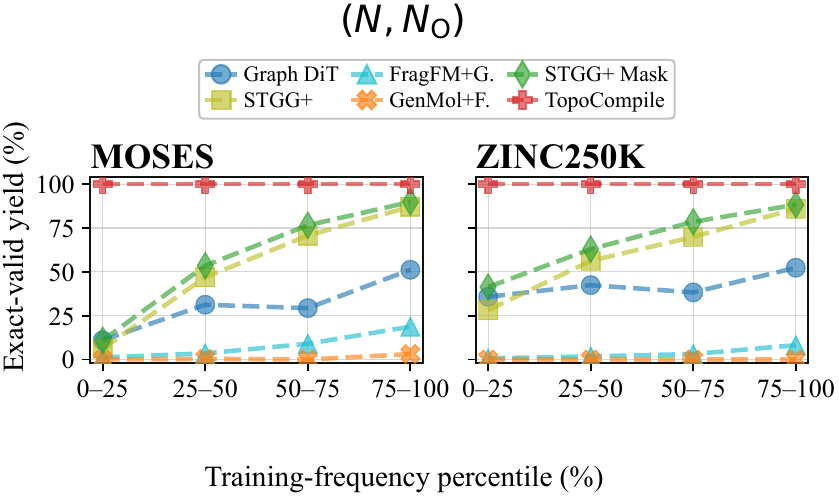}
\refstepcounter{figure}
\label{fig:full-frequency-response-n_oxygen}
\par\small\raggedright
\textbf{Figure~\thefigure:} Exact-valid yield for the $(N,N_{\mathrm{O}})$ schema on MOSES and ZINC250K.
\end{minipage}\hfill
\begin{minipage}[t]{0.485\textwidth}
\centering
\includegraphics[width=\linewidth]{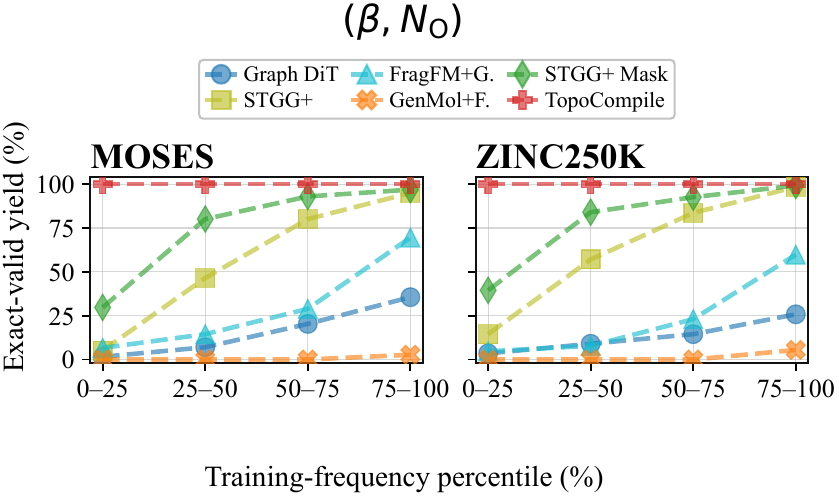}
\refstepcounter{figure}
\label{fig:full-frequency-response-beta_oxygen}
\par\small\raggedright
\textbf{Figure~\thefigure:} Exact-valid yield for the $(\beta,N_{\mathrm{O}})$ schema on MOSES and ZINC250K.
\end{minipage}
\par\vspace{0.18em}

\noindent\hfill
\begin{minipage}[t]{0.485\textwidth}
\centering
\includegraphics[width=\linewidth]{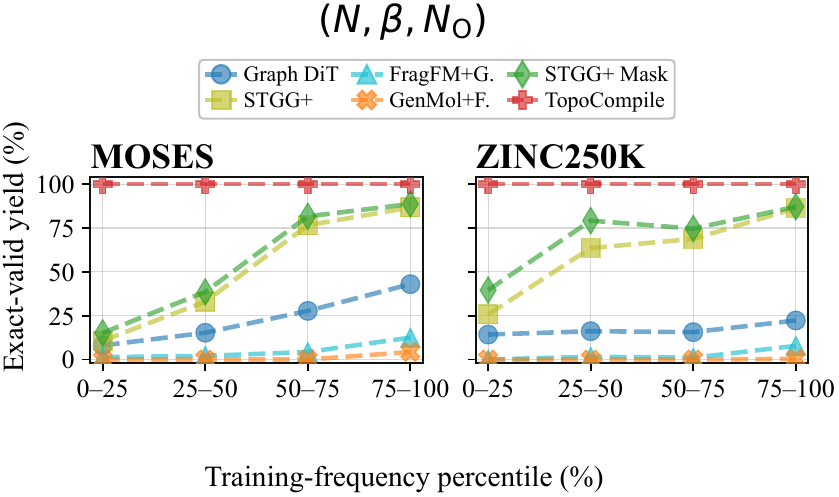}
\refstepcounter{figure}
\label{fig:full-frequency-response-n_beta_oxygen}
\par\small\raggedright
\textbf{Figure~\thefigure:} Exact-valid yield for the $(N,\beta,N_{\mathrm{O}})$ schema on MOSES and ZINC250K.
\end{minipage}
\hfill\null
\par\vspace{0.18em}
\end{samepage}

\paragraph{Results analysis.}
Tables~\ref{tab:exact-control-n}--\ref{tab:four-count-extension} show a
consistent separation between molecular validity and exact delivery.
TopoCompile attains 100.00\% exact-valid yield in every reported
dataset--tier cell for $N$, $\beta$, $(\beta,N_{\rm O})$, and
$(N,\beta,N_{\rm O},R_{\rm sys})$, whereas baseline yields generally decrease
from Common to Rare targets. Across the four available matched-quality
panels, TopoCompile retains novelty of 0.951--1.000 and internal diversity of
0.864--0.896. The completion check preserves exactness across count vectors
while the learned policy distinguishes feasible completions.

The schemas differ in how strongly they restrict baseline success and novel
output multiplicity. Size-only control remains broadly accessible to the
learned baselines, whereas Rare cycle-rank and joint cycle--oxygen requests
produce sharper yield losses, especially on GuacaMol and NPGen. For the
four-count Rare requests on NPGen, STGG+ remains 100.00\% Valid but reaches
3.71\% Exact-valid, while FragFM+Guidance remains 97.46\% Valid and reaches
0.00\% Exact-valid. This gap shows that validity alone does not establish
constraint satisfaction.
Figures~\ref{fig:full-frequency-response-n}--\ref{fig:full-frequency-response-n_beta_oxygen} provide a complementary frequency
breakdown: baseline yields often increase from the lowest to the highest
training-frequency ranges, while TopoCompile remains at 100\% across the
observed schema--frequency cells.
\FloatBarrier

\section{Application: High-Rank Molecule Generation}
\label{app:stress}
Here, high rank refers to a large molecular cycle rank $\beta$. Generating
molecules in this regime expands the tested topology beyond the frequencies
commonly observed during training and requires multiple ring-forming decisions
to satisfy one global structural request. TopoCompile supports this application
because cycle rank is additive over components and the joint DP
removes partial constructions that cannot complete to the requested $\beta$;
the user does not need to specify molecular size $N$ or component count $M$.

The experiments request only $\beta=4,\ldots,11$; no method receives $N$ or
$M$. They test whether each
method continues to satisfy larger cycle-rank requests and whether its valid
outputs become larger as $\beta$ increases. All requests completed with 1,000
attempts per method--rank cell.
\begin{figure}[H]
\centering
\includegraphics[width=\textwidth]{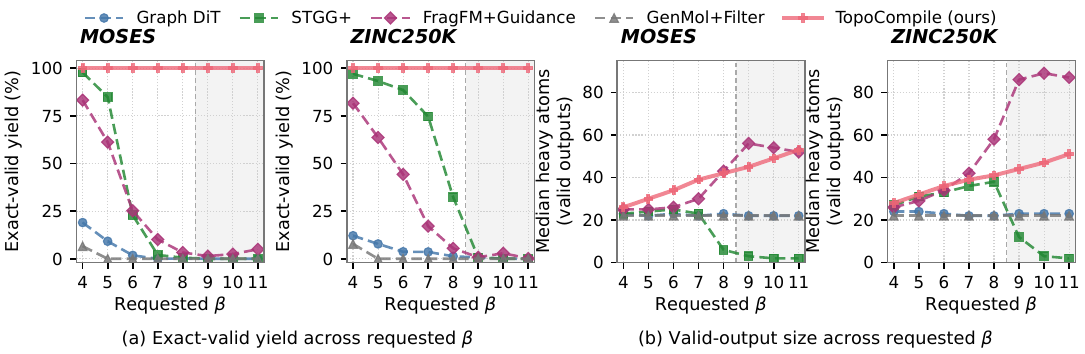}
\caption{Exact-valid yield (a) and valid-output size (b) by cycle rank for
MOSES and ZINC250K.}
\label{fig:high-cycle-robustness}
\end{figure}

TopoCompile maintains 100.0\% exact-valid yield across all eight requested
cycle ranks in both datasets. The baseline yields fall sharply
over the same range, with small nonmonotone fluctuations once they approach
zero. Within the unobserved range $\beta=9,10,11$, the highest baseline yield
is 5.0\% on MOSES and 2.9\% on ZINC250K, both attained by FragFM+Guidance.

TopoCompile's median heavy-atom count increases from 26 to 53 on MOSES and
from 28 to 51 on ZINC250K, while its median component count remains 10 and 11,
respectively. Across these tested requests, higher cycle rank is therefore
associated with larger TopoCompile outputs without an increase in their median
component count. Smaller baseline medians at high $\beta$ generally coincide
with zero or near-zero exact-valid yield; because the size curves include all
valid outputs, they do not indicate more compact exact solutions.

\begin{figure}[H]
\centering
\includegraphics[width=0.82\textwidth]{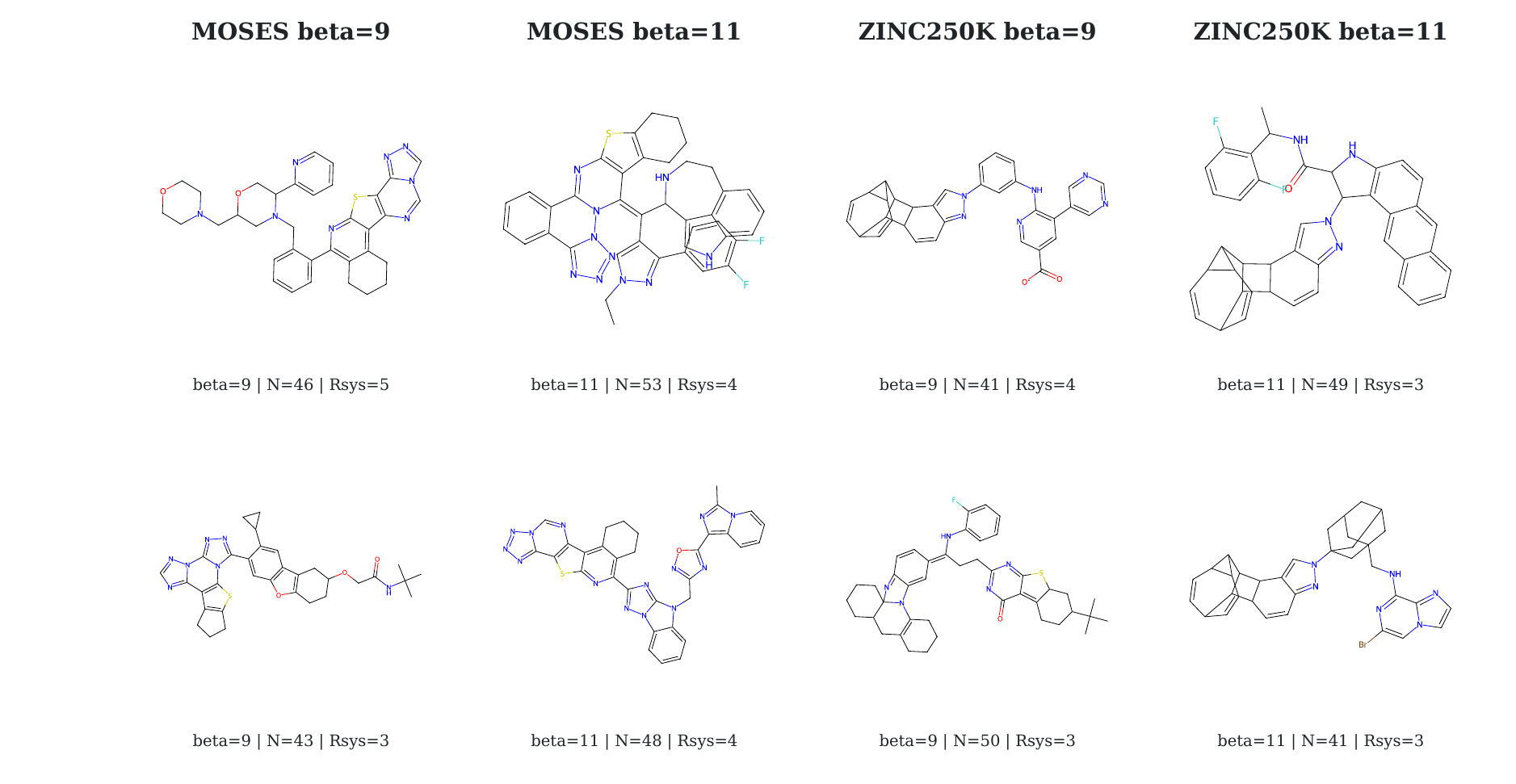}
\caption{TopoCompile examples for unobserved high-cycle requests on MOSES and
ZINC250K, with recomputed structural counts.}
\label{fig:high-rank-gallery}
\end{figure}

Figure~\ref{fig:high-rank-gallery} shows that the same high cycle-rank request
admits visibly different component arrangements. Across the eight fixed
examples, heavy-atom count ranges from 41 to 53 and ring-system count from 3 to
5 while every molecule realizes the requested $\beta\in\{9,11\}$. These
examples illustrate structural variation among exact high-rank outputs;
they do not estimate the output distribution.
\FloatBarrier

\section{Additional Results for Ablation Studies}
\label{app:resource-experiments}

\subsection{Experimental Settings}
We ablate three decisions: the completion check, the sampling policy over
feasible actions, and the compiler's computational reductions. The completion-
checking and policy studies use MOSES and ZINC250K with five schemas:
$(N,\beta)$, $(N,N_{\rm O})$, $(\beta,N_{\rm O})$,
$(N,\beta,N_{\rm O})$, and $(N,\beta,N_{\rm O},R_{\rm sys})$. 

For completion checking, we fix the checkpoint, vocabulary, requests, component
count $M$, target identities, attempt order, Learned policy, and 32 attempts per
target. Local enforces immediate structural and count consistency; Factorized
adds exact count reachability independently of structural reachability; Joint
requires one continuation to satisfy both. Exact-valid yield and novel unique
exact yield (NUE) retain all attempts in their denominators, while NUE removes training
identities and deduplicates within each target.

For the policy study, we fix Joint checking and compare Uniform, Frequency,
and Learned over the first 100 attempts per target. Uniform assigns equal
probability to legal values of each predicted action field, Frequency uses their
training frequencies, and Learned uses the Transformer probabilities. FCD,
internal diversity, scaffold diversity, Filters, and NSPDK use exact-output
pools matched across policies by target and attempt order. The compiler study
compares full compilation with grouping, optional-bound, count-sum, and
structural-pruning ablations on fixed-$M$ cases.

Each compiler variant has three fresh-process repeats; the table reports median
time, states, and transitions, together with peak process memory and completed
compilations.

\subsection{Results}
\paragraph{Completion checking.}
Table~\ref{tab:rq2-mask-schema} reports the completion-checking results.

\renewcommand{\thetable}{S4a}
\begin{table}[t]
\centering
\caption{Completion-checking ablation results for exact-valid yield and novel
unique exact yield on MOSES and ZINC250K.}
\label{tab:rq2-mask-schema}
\scriptsize
\setlength{\tabcolsep}{3.0pt}
\begin{tabular}{@{}llrrr@{}}
\toprule
Dataset & Schema & Local & Factorized & Joint \\
 & & \multicolumn{3}{c}{Exact / NUE (\%)} \\
\midrule
MOSES & $(N,\beta)$ & 41.35 / 39.90 & 72.66 / 70.43 & \textbf{100.00} / \textbf{95.79} \\
 & $(N,N_{\rm O})$ & 30.42 / 29.93 & 83.84 / 83.01 & \textbf{100.00} / \textbf{99.12} \\
 & $(\beta,N_{\rm O})$ & 56.16 / 55.56 & 83.25 / 82.20 & \textbf{100.00} / \textbf{98.78} \\
 & $(N,\beta,N_{\rm O})$ & 34.08 / 33.64 & 83.06 / 82.18 & \textbf{100.00} / \textbf{98.83} \\
 & $(N,\beta,N_{\rm O},R_{\rm sys})$ & 15.04 / 14.84 & 71.68 / 69.82 & \textbf{100.00} / \textbf{97.66} \\
\midrule
ZINC250K & $(N,\beta)$ & 28.03 / 27.73 & 67.14 / 66.60 & \textbf{100.00} / \textbf{98.78} \\
 & $(N,N_{\rm O})$ & 30.27 / 30.13 & 82.67 / 82.18 & \textbf{100.00} / \textbf{99.51} \\
 & $(\beta,N_{\rm O})$ & 50.00 / 50.00 & 74.27 / 74.20 & \textbf{100.00} / \textbf{99.93} \\
 & $(N,\beta,N_{\rm O})$ & 23.68 / 23.29 & 86.23 / 84.62 & \textbf{100.00} / \textbf{97.17} \\
 & $(N,\beta,N_{\rm O},R_{\rm sys})$ & 12.40 / 11.96 & 80.27 / 77.59 & \textbf{100.00} / \textbf{93.51} \\
\bottomrule
\end{tabular}
\end{table}

The table also reports NUE. Its close agreement with Exact-valid yield for
Local and Factorized shows that most exact outputs are absent from training and
unique within their targets.

\paragraph{Sampling policy.}
Figure~\ref{fig:rq2-policy-tradeoff} and
Table~\ref{tab:rq2-policy-dataset-quality}
compare sampling policies under Joint checking.
\begin{figure}[t]
\centering
\includegraphics[width=0.98\textwidth]{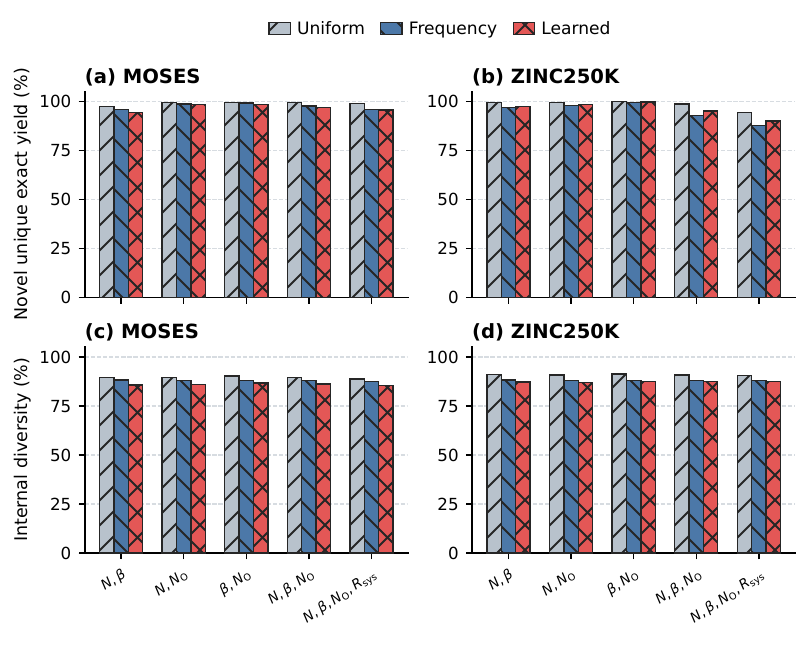}
\caption{Joint ablations: NUE (a,b) and matched-output internal diversity
(c,d).}
\label{fig:rq2-policy-tradeoff}
\end{figure}

\renewcommand{\thetable}{S4c}
\begin{table}[t]
\centering
\caption{Dataset-specific Filters and NSPDK results under different sampling
policies.}
\label{tab:rq2-policy-dataset-quality}
\footnotesize
\setlength{\tabcolsep}{4pt}
\begin{tabular}{@{}lllrrr@{}}
\toprule
Dataset & Schema & Metric & Uniform & Frequency & Learned \\
\midrule
MOSES & $(N,\beta)$ & Filters (\%) $\uparrow$ & 66.39 & 73.97 & \textbf{97.51} \\
 & $(N,N_{\rm O})$ & Filters (\%) $\uparrow$ & 66.91 & 76.27 & \textbf{95.66} \\
 & $(\beta,N_{\rm O})$ & Filters (\%) $\uparrow$ & 65.00 & 77.84 & \textbf{94.74} \\
 & $(N,\beta,N_{\rm O})$ & Filters (\%) $\uparrow$ & 64.66 & 75.19 & \textbf{96.16} \\
 & $(N,\beta,N_{\rm O},R_{\rm sys})$ & Filters (\%) $\uparrow$ & 69.61 & 74.77 & \textbf{96.38} \\
\midrule
ZINC250K & $(N,\beta)$ & NSPDK $\downarrow$ & 0.0739 & 0.0373 & \textbf{0.0239} \\
 & $(N,N_{\rm O})$ & NSPDK $\downarrow$ & 0.0808 & 0.0207 & \textbf{0.0046} \\
 & $(\beta,N_{\rm O})$ & NSPDK $\downarrow$ & 0.0666 & 0.0181 & \textbf{0.0045} \\
 & $(N,\beta,N_{\rm O})$ & NSPDK $\downarrow$ & 0.0791 & 0.0412 & \textbf{0.0242} \\
\bottomrule
\end{tabular}
\end{table}

\paragraph{Compiler computation.}

\renewcommand{\thetable}{S6}
\begin{table}[H]
\centering
\caption{Compiler-computation ablation results on matched fixed-$M$
conditions.}
\label{tab:compiler-ablation}
\footnotesize
\setlength{\tabcolsep}{3pt}
\begin{tabular}{@{}lrrrrr@{}}
\toprule
Dataset / variant & Time (s) & States & Transitions & Peak (GiB) & Complete / total \\\midrule
moses/full & 0.044 & 1.35e+03 & 2.18e+04 & 4.04 & 4314/4314 \\
moses/without grouping & 0.611 & 1.35e+03 & 3.84e+05 & 4.04 & 4314/4314 \\
moses/without optional bounds & 0.0473 & 1.35e+03 & 2.18e+04 & 4.04 & 4314/4314 \\
moses/without count sum & 1.51 & 3.88e+03 & 1.06e+06 & 4.04 & 4314/4314 \\
moses/without structural pruning & 9.4 & 3.32e+04 & 5.56e+06 & 4.04 & 3722/4314 \\
zinc250k/full & 0.0488 & 1.5e+03 & 2.22e+04 & 1.36 & 5448/5448 \\
zinc250k/without grouping & 0.596 & 1.5e+03 & 3.44e+05 & 1.36 & 5448/5448 \\
zinc250k/without optional bounds & 0.0522 & 1.5e+03 & 2.22e+04 & 1.36 & 5448/5448 \\
zinc250k/without count sum & 1.8 & 4.37e+03 & 1.26e+06 & 1.36 & 5442/5448 \\
zinc250k/without structural pruning & 0.495 & 5.64e+03 & 2.57e+05 & 1.36 & 5058/5448 \\

\bottomrule
\end{tabular}
\end{table}

The compiler ablation distinguishes reductions that change the reachable state
set from those that avoid repeated work. Without grouping, the state count does
not change, but transitions increase by 17.6$\times$ on MOSES and
15.5$\times$ on ZINC250K, explaining the corresponding 13.9$\times$ and
12.2$\times$ time increases. Removing count-sum or structural pruning also
expands the state and transition sets and can prevent completion within the
resource limits. Removing optional bounds leaves states, transitions, and
completion unchanged and increases median time by at most 7.5\%, so these
bounds have a limited effect under the tested conditions.
\FloatBarrier


\section{Vocabulary Support and Component Recombination}
\label{app:boundaries}
\renewcommand{\thetable}{S8a}
\begin{table}[H]
\centering
\caption{Component-vocabulary coverage results on the four test sets.}
\label{tab:vocabulary-support}
\footnotesize
\begin{tabular}{lrrrrrr}
\toprule
Dataset & \shortstack{Vocabulary\\size} & Mol. cov. & Component cov. & Whole & Mean $M/N$ & Core occ./unique \\
\midrule
MOSES & 14,798 & 99.29 & 99.93 & 0.0023 & 0.454 & 99.88/82.56 \\
ZINC250K & 10,651 & 97.76 & 99.79 & 0.0200 & 0.476 & 99.67/94.52 \\
GuacaMol & 57,662 & 97.33 & 99.79 & 0.0582 & 0.453 & 99.65/81.58 \\
NPGen & 70,591 & 93.25 & 99.64 & 0.1397 & 0.518 & 98.80/83.21 \\
\bottomrule
\end{tabular}
\end{table}
\renewcommand{\thetable}{S8b}
\begin{table}[H]
\centering
\caption{Ring-distribution results for TopoCompile on MOSES, ZINC250K, and
GuacaMol.}
\label{tab:ring-component-audit}
\footnotesize
\resizebox{\textwidth}{!}{%
\begin{tabular}{lrrrrrrrrr}
\toprule
Dataset & Ring systems & $\beta$ & Ring size & Topology & Core & Core pair & Top core (\%) & Distinct cov. (\%) & Valid/attempts \\
\midrule
MOSES & 0.9634 & 0.9354 & 0.9992 & 0.9976 & 0.9778 & 0.8912 & 33.85 & 19.77 & 25000/25000 \\
ZINC250K & 0.9945 & 0.9946 & 0.9998 & 0.9990 & 0.9689 & 0.8351 & 36.04 & 35.62 & 25000/25000 \\
GuacaMol & 0.9970 & 0.9973 & 0.9988 & 0.9973 & 0.9262 & 0.8157 & 43.67 & 8.29 & 10000/10000 \\
\bottomrule
\end{tabular}
}
\end{table}

Singleton component types account for 43.1\%, 43.8\%, 52.0\%, and 51.2\%
of the MOSES, ZINC250K, GuacaMol, and NPGen vocabularies, respectively.
The 1,000 most frequent types cover 98.9\%, 98.3\%, 97.2\%, and 97.4\%
of training component occurrences. Occurrence coverage and type-level
coverage therefore describe different aspects of the support boundary.
Component coverage defines the support of the representation. New molecules
can combine known components in previously unobserved multisets and bridge
attachments; this does not create a cyclic interior absent from the vocabulary.
For the MOSES sample, 23,146 of 24,978 distinct generated canonical
SMILES are absent from training. Among all 25,000 valid generated molecules,
76.61\% use a new component multiset and 97.96\% a new attachment-bag
signature. The latter records endpoint slots and bond types while ignoring
connectivity among multiple instances of the same component template.

\section{Supplementary Molecules}
\label{app:molecule-visualizations}
\begin{figure}[H]
\centering
\includegraphics[width=0.62\textwidth]{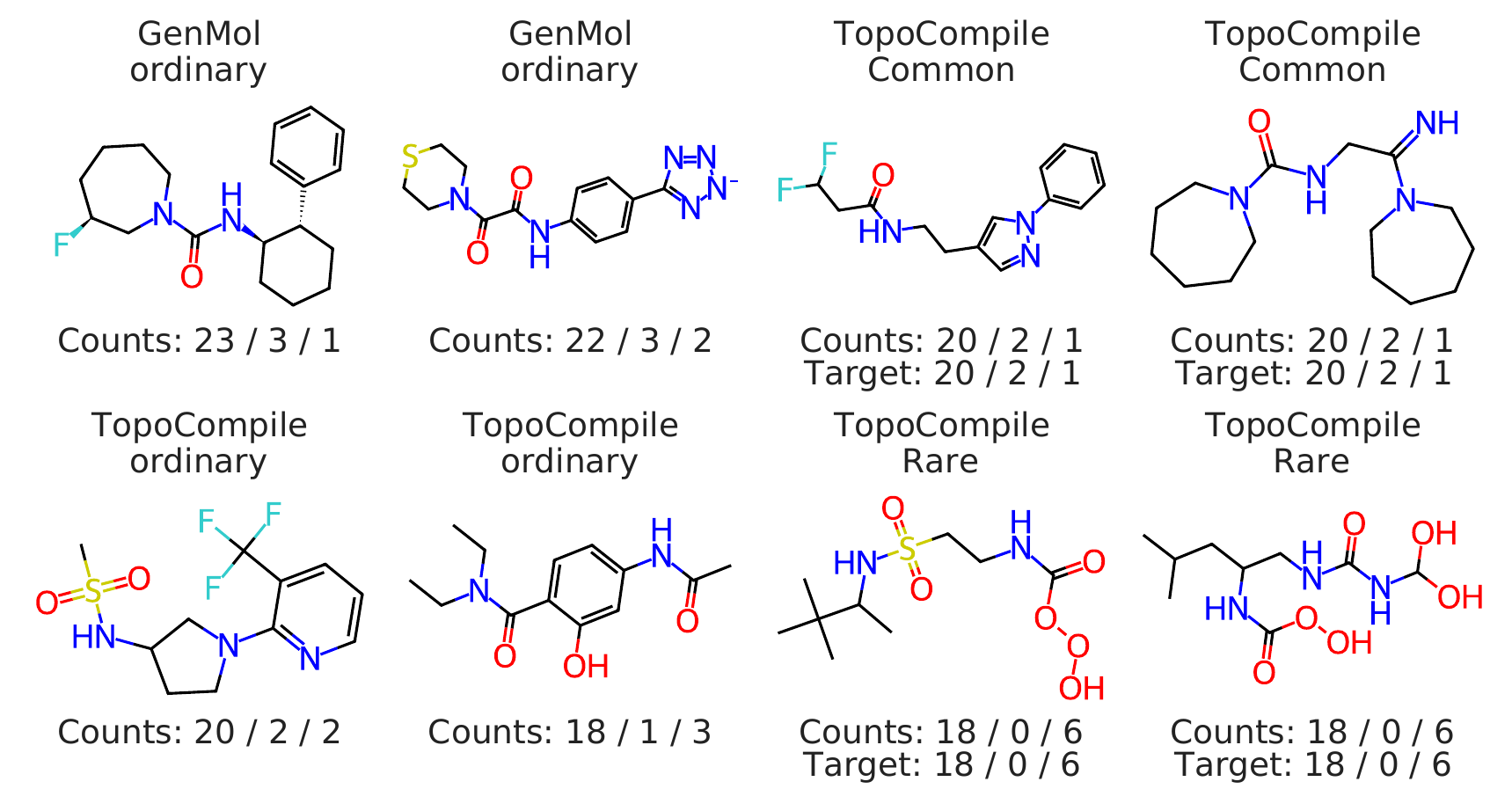}
\caption{MOSES examples from unconditional generation and shared
exact-control requests.}
\label{fig:app-standard-random-gallery}
\end{figure}
\FloatBarrier

\section{Relation to Hard-Constrained Molecular Generation}
\label{app:hard-constrained}

TopoCompile addresses a restricted but exact problem: equalities over
non-negative component-additive integer counts under a finite component
vocabulary and typed attachment grammar. Table~\ref{tab:hard-constrained-comparison}
separates this contribution from exact constrained generation in general.
Projection methods enforce supported graph properties during diffusion; CP-based
methods solve or propagate symbolic sequence models; and SMT-based methods
assemble neural proposals under logical constraints. TopoCompile instead
specializes completion reasoning to the aggregate state
$(k,\mathbf{f},\mathbf{b})$ and queries exact suffix existence before every
concrete component-and-attachment action. Its claim is therefore not the first
use of hard constraints or neuro-symbolic generation, but an exact and reusable
prefix-completion test induced by the component-tree representation.

Classical constraint programming constructs molecular graphs with required
and forbidden fragments \citep{omrani2020constraints}. Beyond molecules,
Ctrl-G combines a distilled hidden Markov model with deterministic finite
automata to guarantee logical constraints during text generation
\citep{zhang2024adaptable}. TopoCompile instead derives a Boolean completion
test directly from molecular count and attachment requirements, without a
distilled sequence model. MacroGuide uses persistent-homology guidance to
promote macrocycle generation, addressing a different topological objective
from exact cycle-rank control \citep{maksymiuk2026macroguide}.

\begingroup
\renewcommand{\thetable}{S10}
\begin{table}[H]
\centering
\caption{Mechanistic comparison of hard-constrained generators.}
\label{tab:hard-constrained-comparison}
\tiny
\setlength{\tabcolsep}{1.5pt}
\renewcommand{\arraystretch}{0.84}
\begin{tabularx}{\textwidth}{@{}>{\raggedright\arraybackslash}p{0.18\textwidth}>{\raggedright\arraybackslash}X>{\raggedright\arraybackslash}X>{\raggedright\arraybackslash}X@{}}
\toprule
Method & Specification & Enforcement point & Relation to \textsc{TopoCompile} \\
\midrule
PRODIGY / ConStruct\newline
\citep{sharma2024prodigy,madeira2024construct}
& Supported hard graph properties
& Projection during diffusion
& Exactness depends on a suitable property-specific projector or closure
  condition, rather than suffix feasibility for an additive count vector \\
\addlinespace
Saikali \& Pesant\newline
\citep{saikali2026constrained}
& SMILES grammar with physicochemical and structural constraints
& CP construction and search, guided by belief propagation
& General symbolic search rather than a compact component-tree completion
  state \\
\addlinespace
GeAI-BLAnC\newline
\citep{manibod2025constrained}
& Fixed-length sequence constraints, including SMILES grammar and
  molecular-weight ranges
& CP/BP modifies the next-token distribution
& Non-backtracking generation may fail to complete; full completion can use
  backtracking \\
\addlinespace
Geng et al.\newline
\citep{geng2026neural}
& Chemical, structural, and user logical constraints
& SMT assembly from neural proposals
& The solver guarantees assembled outputs rather than every neural
  construction prefix \\
\addlinespace
\textsc{TopoCompile}
& Multiple component-additive count equalities
& Exact DP on $(k,\mathbf{f},\mathbf{b})$ before every action
& Every retained prefix is completable relative to the vocabulary and grammar,
  without trajectory backtracking or completed-sample rejection/repair \\
\bottomrule
\end{tabularx}
\end{table}
\endgroup

\FloatBarrier

The published evaluations of these methods are not numerically interchangeable
with Table~\ref{tab:shared-exact-control}. GeAI-BLAnC studies fixed-length SMILES
under a molecular-weight interval, projection methods use their own constraint
families and diffusion checkpoints, and Geng et al.~\citep{geng2026neural}
evaluate SMT-backed logical constraints. We therefore compare these systems
mechanistically and reserve Table~\ref{tab:shared-exact-control} for methods
evaluated under the shared four-count requests and attempt budgets.
\FloatBarrier

\end{document}